\documentclass{article}
\usepackage{svaiter}
\usepackage{algorithm2e}

\usepackage[textwidth=2.5cm, textsize=tiny]{todonotes}

\RequirePackage[colorlinks,citecolor=blue,urlcolor=blue]{hyperref}
\usepackage{color}

\usepackage{subcaption}

\usepackage[margin=1.2in]{geometry}

\def\Aexp{A^{\text{exp}}}
\def\Aunif{A^{\text{unif}}}
\def\Asmooth{A^{\text{smooth}}}

\def\Psmooth{P^{\text{smooth}}}
\def\rhoPZ{\rho^{\textup{(PZ)}}}

\def\betamax{\beta_{\max}}
\def\Cbeta{C_\beta}
\def\ntaumax{\bar{n}_{\max}}
\def\ntaumin{\bar{n}_{\min}}
\def\commratio{\mu_B}
\def\CCbeta{C'_\beta}
\def\err{E}

\newcommand{\order}[1]{\mathcal{O}\left(#1\right)}
\newcommand{\pa}[1]{\left(#1\right)}
\renewcommand{\norm}[1]{\left\lVert#1\right\rVert}

\def\figsize{0.4}

\title{Sparse and Smooth:\\ improved guarantees for Spectral Clustering\\ in the Dynamic Stochastic Block Model}

\author{
    Nicolas Keriven\\
    CNRS \& GIPSA-lab\\
    11 rue des Math\'ematiques, 38400 St-Martin-d'H\'er\`es, France. \\
    nicolas.keriven@gipsa-lab.grenoble-inp.fr
    \and
    Samuel Vaiter\\
    CNRS \& IMB, Universit\'e de Bourgogne \\
    9 avenue Alain Savary, 21000 Dijon, France.  \\
    samuel.vaiter@u-bourgogne.fr
}
\date{}
\begin{document}

\maketitle

\begin{abstract}
In this paper, we analyse classical variants of the Spectral Clustering (SC) algorithm in the Dynamic Stochastic Block Model (DSBM). Existing results show that, in the relatively sparse case where the expected degree grows logarithmically with the number of nodes, guarantees in the static case can be extended to the dynamic case and yield improved error bounds when the DSBM is sufficiently \emph{smooth} in time, that is, the communities do not change too much between two time steps. We improve over these results by drawing a new link between the sparsity and the smoothness of the DSBM: the more regular the DSBM is, the more sparse it can be, while still guaranteeing consistent recovery. In particular, a mild condition on the smoothness allows to treat the \emph{sparse} case with bounded degree. We also extend these guarantees to the normalized Laplacian, and as a by-product of our analysis, we obtain to our knowledge the best spectral concentration bound available for the normalized Laplacian of matrices with independent Bernoulli entries.
\end{abstract}

\section{Introduction}

In recent years, the study of dynamic networks has appeared as a topic of great interest to model complex phenomenons that evolve with time, such as interactions in social networks, the spread of infectious diseases or opinions, or information packets in computer networks. In light of this, many random graphs models, traditionally static (non-dynamic), have been extended to the dynamic case, see~\cite{Goldenberg2009, Holme2015} for reviews. One of the most popular use of dynamic networks consists in detecting and tracking communities of well-connected nodes, for instance users of a social network~\cite{Yang2011b, Xu2014a, Wang2017}. In this context, the classical Stochastic Block Model (SBM)~\cite{Holland1983}, in which nodes intra- and inter-communities are linked independently with some prescribed probabilities, has been extended to dynamic settings (DSBM) in a myriad of ways. In this paper, we consider one of the first (and most popular) extension~\cite{Yang2011b} as a discrete Hidden Markov Model (HMM) as well as one of its simplification, where node memberships follow a Markov chain with respect to time, and connections are generated by a classical SBM conditionally on the memberships. We will also consider a slight simplification as in~\cite{Pensky2019}, where the authors remove the Markov Chain assumption and consider deterministic community memberships at each time steps. Many other models have been proposed since, to take into account evolving connection probabilities~\cite{Xu2014a, Matias2017, Pensky2019}, varying number of nodes~\cite{Xu2014}, connections that depend on their previous states~\cite{Xu2014}, mixed-membership SBM~\cite{Ho2011} or multi-graphs~\cite{Han2015}. 

The literature on clustering nodes in a graph is vast, with a variety of methods. Arguably, the most popular class of algorithms in practice is that of \emph{spectral clustering} (SC) methods~\cite{Ng2001, VonLuxburg2007}, which consist in applying a classical clustering algorithm for vectorial data, often the well-known $k$-means algorithm~\cite{Lloyd1982}, to the eigenvectors of a matrix related to the structure of the graph such as the adjacency matrix or normalized Laplacian. In a dynamic context, we will consider in this paper one of the simplest adaptation of SC, which consists in feeding a version of the adjacency matrix \emph{smoothed in time} to the classical SC algorithm, in hope of implicitely enforcing smoothness of the communities. This can be an averaged version of the adjacency matrix over a finite window~\cite{Ho2011, Pensky2019}, or computed through recursive updates with a certain ``forgetting factor''~\cite{Chi2007, Chi2009, Xu2010}, which is somehow more amenable to streaming computing. Other works explicitely enforce smoothness between the communities or between the eigenvectors considered in SC through efficient updates~\cite{Ning2007,Dhanjal2014,Liu2018}.

Beyond SC, many other methods have been proposed, such as Maximum Likelihood or variational approaches, which are consistent for the SBM and DSBM \cite{Celisse2012, Matias2017, Longepierre2019}, Bayesian approaches~\cite{Yang2011b}, learning-based approaches \cite{Bach2006}, or neural networks~\cite{Bruna2017}. Many variants of the SC itself exist, often to accelerate computation \cite{Tremblay2019}.

\paragraph*{Guarantees for Spectral Clustering}

There is a vast literature on the theoretical analysis of SC, and guarantees come in many different flavors. Several works analyze the algorithm when the graph is well-clustered (in some sense) \cite{Peng2014}, in terms of spectral convergence of the normalized Laplacian when the number of nodes goes to infinity \cite{VonLuxburg2008, Diao2016a, GarciaTrillos2018, Tang2018a}, or using random matrix theory~\cite{Couillet2016a}.

It is well-known that a key quantity to analyse SC algorithms is the density of edges with respect to the number of nodes. In the specific case of independent Bernoulli edges like the SBM and DSBM, this correspond to the mean probability of connection, which will be denoted by $\alpha_n$ in this paper, where $n$ is the number of nodes in the graph. The \emph{dense} case $\alpha_n \sim 1$ is generally trivial to analyse~\cite{Le2018}. At the other end of the spectrum, the so-called \emph{sparse} case $\alpha_n\sim \frac{1}{n}$ is much more complex, since the graph is not even guaranteed to be connected with high probability~\cite{Abbe2018}.

Modern analyses of the sparse case are often inspired by statistical physics~\cite{Krzakala2013, Mossel2017, Abbe2018}, and are interested with the computation of a \emph{detectability threshold}, that is, the characterization of regimes of parameters in which there exists (or not) an algorithm that asymptotically performs better than random guess. However, this approach does not concern the classic SC algorithm (which will generally fail \cite{Krzakala2013}), and the case where the number of communities $K$ is larger than $2$ is still largely open. In the dynamic case, a conjecture on the detectability threshold is given in \cite{Ghasemian2016}. In parallel, other works study the sparse case by regularizing the adjacency matrix or normalized Laplacian of the graph before the SC algorithm \cite{Le2017a, Le2018}.

In \cite{Lei2015}, Lei and Rinaldo provide strong, non-asymptotic consistency guarantees for the classic SC algorithm on the adjacency matrix (without regularization) in the \emph{relatively sparse} case $\alpha_n \gtrsim \frac{\log n}{n}$, showing that the proportion of misclassified nodes tends to $0$ with a probability that goes to $1$ when the number of nodes $n$ increases. Their recovery results 
are valid for any $K$, potentially growing slowly with $n$. In~\cite{Pensky2019}, Pensky and Zhang extend this analysis to a particular Dynamic SBM, referred to as ``deterministic'' DSBM in the sequel, for the SC algorithm applied to a smoothed adjacency matrix. In this case, another key quantity is the temporal \emph{regularity} of the model $\varepsilon_n$, that is, the proportion of nodes that may change community between two time steps (the smaller $\varepsilon_n$ is, the more regular the model). They showed that, in the relatively sparse case, if the model was sufficiently regular in $\varepsilon_n = o\pa{\frac{1}{\log n}}$, then the error bound of the static case can be improved. However their analysis still takes place in the relatively sparse case even when $\varepsilon_n$ is very low. 


\paragraph*{Contributions} In this paper, we follow the analyses of \cite{Lei2015} and \cite{Pensky2019} and significantly extend them in several ways.
\begin{itemize}
\item[--] Our main contribution is to draw a new link between the sparsity $\alpha_n$ and regularity $\varepsilon_n$ in the analysis of the DSBM: we show that, the more regular the model, the sparser it can be, while still guaranteeing consistency. In particular, a mildly strengthened condition $\varepsilon_n \sim \frac{1}{\log^2 n}$ allows to \emph{give consistent guarantees in the sparse case} $\alpha_n \sim \frac{1}{n}$.
\item[--] We extend the analysis to the normalized Laplacian, which was left open by Lei and Rinaldo~\cite{Lei2015}. As a by-product, in the static case, we obtain, to our knowledge, the best spectral concentration bound available $\norm{L(A) - L(\mathbb{E}(A))} \lesssim \frac{1}{\sqrt{\log n}}$ in the relatively sparse case $\alpha_n \sim \frac{\log n}{n}$.
\item[--] We also improve the rate of the error bounds with respect to the number of communities $K$ when the probabilities of connection between communities decrease with $K$, in both the static \cite{Lei2015} and dynamic \cite{Pensky2019} cases.
\item[--] Finally, we extend our results to the Markov DSBM introduced in \cite{Yang2011b}, and the SC algorithm with an ``exponentially smoothed'' matrix, used in \cite{Chi2007, Xu2010} and appropriate in a streaming computing framework.
\end{itemize}

\paragraph*{Outline} In Section \ref{sec:notation}, we introduce notations, the SBM and DSBM, and recall the SC algorithm. In Section \ref{sec:reco}, we draw a link between recovery guarantees for SC and the concentration of the input matrix in spectral norm, similar to \cite{Lei2015} but extended to the normalized Laplacian. In Section \ref{sec:adj} and \ref{sec:Lapl}, we expose our main concentration results respectively for the adjacency matrix and normalized Laplacian. Proofs are given in Section \ref{sec:proof}, with technical computations deferred to the Appendix.

\section{Framework and notations}\label{sec:notation}

The set of the first $n$ integers is denoted by $[n] = \{1,\ldots,n\}$.
For any vector $d \in \RR^n$, we define $\diag(d) \in \RR^{n \times n}$ to be the diagonal matrix whose elements are given by $d$. For a varying parameter $\alpha_n$, the notation $\alpha_n \sim f(n)$ indicates that, as $n\to \infty$, the quantity $\alpha_n/f(n)$ tends to a non-zero constant, $\alpha_n \lesssim f(n)$ indicates that there is a universal constant $C$ such that $\alpha_n \leq Cf(n)$, and similarly for $\alpha_n \gtrsim f(n)$.

An undirected graph $G = (V,E)$ is formed by a set of nodes $V$ and edges $E \subset V \times V$. For a graph with $n$ nodes, we often adopt $V = [n]$, and we define its (symmetric) adjacency matrix $A \in \set{0,1}^{n\times n}$ such that for $i,j \in [n]$,
\begin{equation*}
  A_{ij} =
  \begin{cases}
    1 & \text{if } \set{i,j} \in E, \\
    0 & \text{otherwise.}
  \end{cases}
\end{equation*}
We also define the (diagonal) degree matrix $D(A)$ by
\begin{equation*}
  D(A) = \diag\left( (d_i)_{i=1}^n \right)
  \qwhereq
  d_i = \sum_{j=1}^n A_{ij} .
\end{equation*}
For any symmetric matrix $A$ such that $\sum_j A_{ij}\neq 0$ for all $i$, the normalized Laplacian $L(A)$ is defined as
\[
  L(W) = D(A)^{-\frac12} A D(A)^{-\frac12} .
\]
We note that, typically, the normalized Laplacian is defined as the matrix $\Id - D(A)^{-\frac12} A D(A)^{-\frac12}$. However, SC is mainly concerned with the eigenvectors of the Laplacian, which are the same for both variants.

\paragraph*{Stochastic Block Model}
Let us start by introducing the classical static SBM. We take the following notations: $n$ the number of nodes, $K$ the number of communities. Each node belongs to exactly one community. We denote by $\Theta \in\set{0,1}^{n\times K}$ the $0-1$ matrix representing the memberships of nodes, where for each node $i$, $\Theta_{ik} = 1$ indicates that it belongs to the $k$th community, and is $0$ otherwise. The (symmetric) adjacency matrix is denoted by $A\in \set{0,1}^{n\times n}$.
Given $\Theta$, for $i<j$ we have
\[
A_{ij}~|~\set{\Theta_{ik} = 1, \Theta_{j\ell}=1} \sim \text{Ber}(B_{k\ell}) \\
\]
where $B\in [0,1]^{K\times K}$ is a symmetric connectivity matrix, and $\text{Ber}(p)$ indicates a Bernoulli random variable with parameter $p$. We also let $A_{ii}=0$ and $A_{ji} = A_{ij}$. Finally, we define $P = \Theta B \Theta^\top \in \RR^{n\times n}$ the matrix storing the probabilities of connection between two nodes off its diagonal, and we have
\begin{equation*}
\mathbb{E}(A) = P - \diag(P)
\end{equation*}
Typically, $B$ has high diagonal terms and low off-diagonal terms. We will consider $B$ of the form
\begin{equation}\label{eq:Bpredef}
B = \alpha_n B_0
\end{equation}
for some  $\alpha_n \in (0,1)$ and $B_0 \in [0,1]^{K\times K}$ whose elements are denoted by $b^{(0)}_{k\ell}$. It is known that the rate $\alpha_n$ when $n \to \infty$ is the main key quantity when analyzing the properties of random graphs. Typical settings include $\alpha_n \sim 1$ (dense graphs), $\alpha_n \sim 1/n$ (sparse graphs), or middle grounds such as $\alpha_n \sim \frac{\log n}{n}$, usually referred to ``relatively sparse'' graphs. As we will see, it is known that strong guarantees of consistency can be given in the relatively sparse case, while the sparse case it hard to analyze and only partially understood.

For some maximum and minimum community sizes $n_{\max}\geq \frac{n}{K}$ and $n_{\min}\leq \frac{n}{K}$, we define the set of admissible community sizes $N\eqdef \{(n_k)_{k=1}^K ~|~ n_{\min} \leq n_k \leq n_{\max}, \sum_k n_k = n\}$, and
\begin{align}\label{eq:ntaudef}
\ntaumax \eqdef \max_{(n_\ell)_\ell \in N, k\leq K} \sum_{\ell} n_\ell b_{k\ell}^{(0)}
,\quad\ntaumin \eqdef \min_{(n_\ell)_\ell \in N, k\leq K} \sum_{\ell} n_\ell b_{k\ell}^{(0)}
\end{align}
These quantities are such that the expected degree will be comprised between $\alpha_n \ntaumin$ and $\alpha_n \ntaumax$.
For simplicity, we will sometimes express our results with $B_0$ equal to:
\begin{equation}\label{eq:Bdef}
B_0 =(1-\tau) \Id_K + \tau 1_K 1_K^\top
\end{equation}
In other words, $B$ contains $\alpha_n$ on its diagonal and $\tau\alpha_n$ outside.
For this expression of $B_0$, we have $\ntaumax = (1-\tau) n_{\max} + n \tau$, and similarly for $\ntaumin$. Interestingly, in the case of balanced communities $n_{\max}, n_{\min} \sim \frac{n}{K}$, we have then
\[
\ntaumin, \ntaumax \sim \begin{cases}
n &\text{if $\tau \sim 1$} \\
\frac{n}{K} &\text{if $\tau \sim \frac{1}{K}$}
\end{cases}
\]

\paragraph*{Dynamic SBM} The Dynamic SBM (DSBM) is a random model for generating adjacency matrices $A_0,\ldots,A_t$ at each time step. Each $A_i$ will be generated according to a classical SBM with constant number of nodes $n$, number of communities $K$ and connectivity matrix $B$, but changing node memberships $\Theta_t$. Note that several works consider changing number of nodes~\cite{Xu2014} or changing connectivity matrix~\cite{Pensky2019}, but for simplicity we assume that they are constant in time here.
We will consider two potential models on the $\Theta_t$.
\begin{itemize}
\item[--] The simplest one, adopted in~\cite{Pensky2019}, is to consider that $\Theta_0,\ldots,\Theta_t$ are deterministic variables. In this case, we will assume that only a number $s\leq n$ of nodes change communities between each time step $t-1$ and $t$, and denote $\epsilon_n = s/n$ this relative proportion of nodes. We will also assume that at all time steps, the communities sizes are comprised between some $n_{\min}$ and $n_{\max}$, which will typically be of the order of $n/K$ for balanced communities. 
As a shorthand, we will simply refer to this model as \emph{deterministic} DSBM (keeping in mind that the $A_t$ are still random).
\item[--] In the second model, similar to~\cite{Yang2011b} we assume that the nodes memberships follow a Markov chain, such that between two time steps, all nodes have a probability $1-\varepsilon_n$ to stay in the same community, and $\varepsilon_n$ to go into any other community, that is:
\begin{align*}
&\forall i,k,~ \mathbb{P}\pa{(\Theta_t)_{ik} = 1~|~(\Theta_{t-1})_{ik} = 1} = 1-\varepsilon_n, \\
&\forall \ell \neq k,~ \mathbb{P}\pa{(\Theta_t)_{i\ell} = 1~|~(\Theta_{t-1})_{ik} = 1} = \frac{\varepsilon_n}{K-1}
\end{align*}
Then, conditionally on $\Theta_t$, the $A_t$ are drawn independently according to a SBM. The global model is thus a Hidden Markov Model (HMM).
We will simply refer to this case as Markov DSBM. Note that, in this case, it is rather difficult to quantify, in a non-asymptotic manner, the probability of having bounded community sizes globally holding for all time steps. Hence $\ntaumax,\ntaumin$ will not intervene in our analysis of this case.
\end{itemize}

\paragraph*{Goal and error measure} The goal of a clustering algorithm is to give an estimator $\hat \Theta$ of the node memberships $\Theta$, up to permutation of the communities labels. 
We consider the following measure of discrepancy between $\Theta$ and an estimator $\hat \Theta$~\cite{Lei2015}:
\begin{equation}
\label{eq:error_measure}
E(\hat \Theta, \Theta) = \min_{Q\in\mathcal{P}_k} \frac{1}{n} \norm{\hat \Theta Q - \Theta}_0
\end{equation}
where $\mathcal{P}_k$ is the set of permutation matrices of $[k]$ and $\norm{\cdot}_0$ counts the number of non-zero elements of a matrix. While other error measures are possible, as we will see one can generally relate them to a spectral concentration property, which will be the main focus of this paper. 

In the dynamic case, a possible goal is to estimate $\Theta_1,\ldots, \Theta_t$ for all time steps simultaneously~\cite{Xu2010, Pensky2019}. Here we consider a slightly different goal: \emph{at a given time step $t$}, we seek to estimate $\Theta_t$ with the best precision possible, by exploiting past data. In general, this will give rise to methods that are computationally lighter than simultaneous estimation of all the $\Theta_t$'s, and more amenable to streaming computing, where one maintains an estimator without having to keep all past data in memory. Naturally, such methods could be applied independently at each time step to produce estimators of all the $\Theta_t$'s, but this is not the primary goal here.

\paragraph*{Spectral Clustering (SC) algorithm}
Spectral Clustering~\cite{Ng2001} is nowadays one of the leading methods to identify communities in an unsupervised setting.
The basic idea is to solve the $K$-means problem~\cite{Lloyd1982} on the $K$ leading eigenvectors $E_K$ of either the adjacency matrix or (normalized) Laplacian.
Solving the $K$-means, i.e., obtaining 
\begin{equation}\label{eq:kmeans-functional}
  (\bar \Theta, \bar C) \in 
  \uArgmin{\Theta \in \RR^{n \times K}, C \in \RR^{K \times K}}
  \norm{\Theta C - E_K}_F^2 ,
\end{equation}
is known to be NP-hard, but several approximation algorithms, such as~\cite{Kumar2004}, are known to produce $1+\delta$ approximate solutions $(\hat \Theta, \hat C)$
\begin{equation*}
  \norm{\hat \Theta \hat C - E_K}_F^2 \leq (1+\delta) \norm{\bar \Theta \bar C - E_K}_F^2 .
\end{equation*}
The SC is summarized in Algorithm~\ref{alg:SC}.
\begin{algorithm}
\KwData{Matrix $M \in \RR^{n \times n}$ (typically adjacency or normalized Laplacian), number of communities $K$, approximation ratio $\delta > 0$}
\KwResult{Estimated communities $\hat \Theta \in \RR^{n \times K}$}
Compute the $K$ leading eigenvectors $E_K$ of $M$.

Obtain a $(1+\delta)$-approximation $(\hat \Theta, \hat C)$ of \eqref{eq:kmeans-functional}.

Return $\hat \Theta$.
\caption{Spectral Clustering algorithm}\label{alg:SC}
\end{algorithm}

In the dynamic case, a typical approach to exploit past data is to replace the adjacency matrix $A_t$ with a version ``smoothed'' in time $A^\text{smooth}_t$, and feed either $\hat P = A^\text{smooth}_t$ or the corresponding Laplacian $\hat L = L(A^\text{smooth}_t)$ to the classical SC algorithm.
In~\cite{Pensky2019}, the authors consider the smoothed adjacency matrix as an average over its last $r$ values:
\begin{equation}\label{eq:estimatorUnif}
\Aunif_t = \frac{1}{r} \sum_{k=0}^{r-1} A_{t-k} .
\end{equation}
Note that, in the original paper, the authors sometimes consider non-uniform weights due to potential changes in time of the connectivity matrix $B_t$, but in our case we consider a fixed $B$, and thus uniform weights $\frac{1}{r}$.
In this paper, we will also consider the ``exponentially smoothed'' estimator proposed by \cite{Chi2007,Chi2009,Xu2011a}, which is computed recursively as:
\begin{equation}\label{eq:estimatorExp}
\Aexp_t = (1-\lambda) \Aexp_{t-1} + \lambda A_t .
\end{equation}
for some ``forgetting factor'' $\lambda \in (0,1]$, and $\Aexp_0=A_0$. Compared to the uniform estimator \eqref{eq:estimatorUnif}, this kind of estimator is somewhat more amenable to streaming and online computing, since only the current $\Aexp_t$ needs to be stored in memory instead of the last $r$ values $A_t,A_{t-1},\ldots,A_{t-r+1}$ (note however that $\Aexp_t$ may be denser that a typical adjacency matrix, so the memory gain is sometimes mitigated depending on the case).

In Fig. \ref{fig:basic}, we illustrate the performance of the SC algorithm on a synthetic DSBM example. As expected, the normalized Laplacian $L(\Aexp_t)$ generally performs better than $\Aexp_t$. Interestingly, the optimal forgetting factor $\lambda$ is slightly different from one to the other, and the normalized Laplacian reaches a higher performance altogether. We then compare $\Aunif_t$ and $\Aexp_t$. As we will see in the sequel, taking $r \sim \frac{1}{\lambda}$ often results in the same performance for both estimators. However, a clear advantage of the exponential estimator is that it is not limited to discrete window sizes, but has a continuous forgetting factor. As such, $\Aexp_t$ with the optimal $\lambda$ often reaches a better performance than $\Aunif_t$ with the optimal $r$.

\begin{figure}[t] 
  \begin{subfigure}[t]{\figsize\linewidth}
    \centering
    \includegraphics[width=1.0\linewidth]{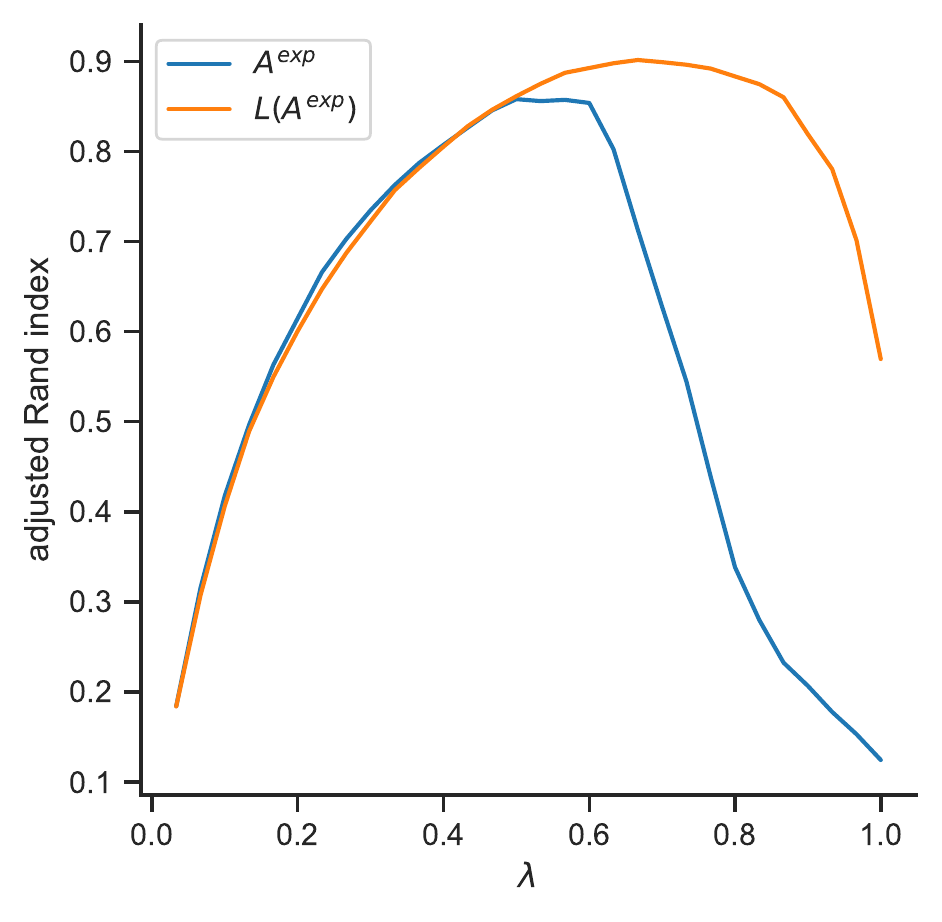} 
    \caption{Adjusted Rand index in function of the forgetting factor $\lambda$ for the adjacency matrix and the normalized Laplacian.} 
    \label{fig:basic:AvsL} 
    \vspace{4ex}
  \end{subfigure}
  \hspace{1em}
  \begin{subfigure}[t]{\figsize\linewidth}
    \centering
    \includegraphics[width=1.0\linewidth]{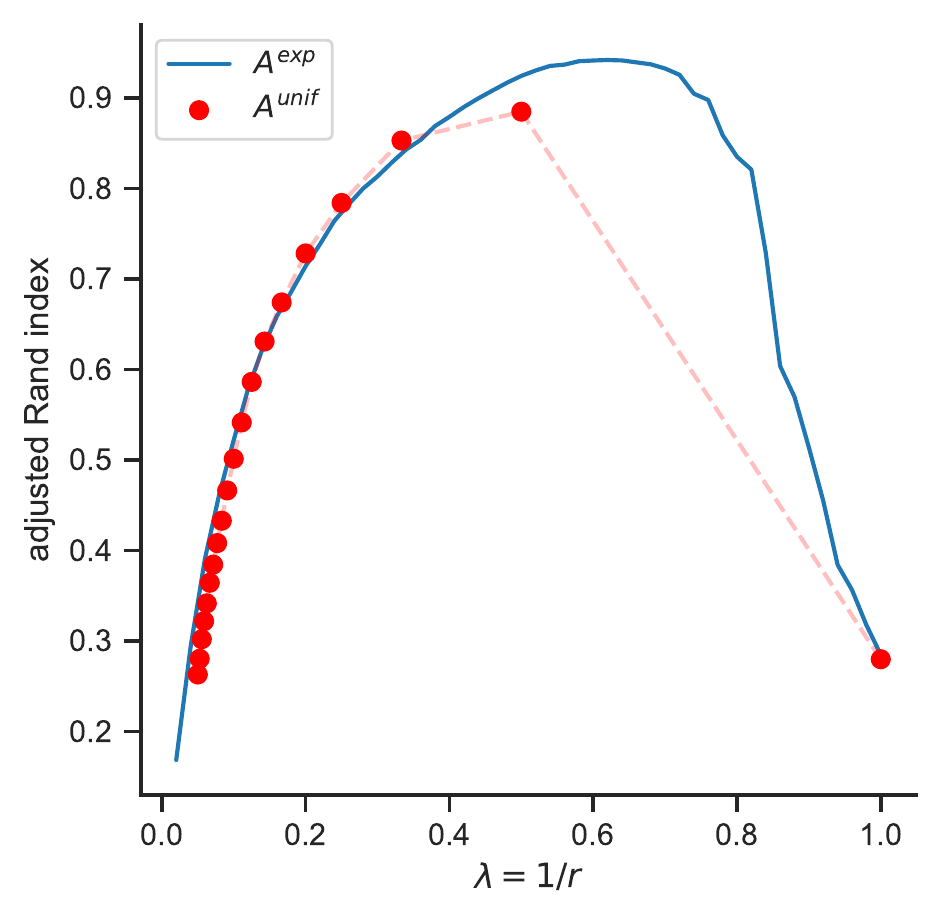} 
    \caption{Adjusted Rand index in function of the forgetting factor $\lambda$ for $\Aexp_t$ and the window size $r$ for $\Aunif_t$.}
    \label{fig:AvsL:unifVSexp} 
    \vspace{4ex}
  \end{subfigure} 
  \caption{Performance results for SC on synthetic data.}
  \label{fig:basic} 
\end{figure}

\section{From Spectral Clustering to spectral norm concentration}\label{sec:reco}




As described in~\cite{Lei2015}, a key quantity for analyzing SC algorithm is the concentration of the adjacency matrix around its expectation \emph{in spectral norm}. As a first contribution, we prove the following lemma, which is a generalisation of this result to the normalized Laplacian.
\begin{lemma}\label{lem:reco}
Let $P=\Theta B \Theta^\top$ correspond to some SBM with $K$ communities, where $n_{\max}$, $n'_{\max}$ and $n_{\min}$ are respectively the largest, second-largest and smallest community size. Assume $B=\alpha_n B_0$ for any $B_0$ with smallest eigenvalue $\gamma$. Let $\hat P$ be an estimator of $P$, and $\hat \Theta$ be the output of Algorithm~\ref{alg:SC} on $\hat P$ with a $(1+\delta)$-approximate $k$-means algorithm. Then
\begin{equation}\label{eq:guaranteeP}
E(\hat \Theta, \Theta) \lesssim (1+\delta) \frac{n'_{\max} K}{n \alpha_n^2 n_{\min}^2 \gamma^2}\norm{\hat P - P}^2\, ,
\end{equation}
Similarly, if $\hat L$ is an estimator of $L(P)$ and $\hat \Theta$ is the output of Algorithm~\ref{alg:SC} on $\hat L$, it holds that
\begin{equation}\label{eq:guaranteeL}
E(\hat \Theta, \Theta) \lesssim (1+\delta) \frac{n'_{\max} K \ntaumax^2}{n n_{\min}^2 \gamma^2}\norm{\hat L - L(P)}^2\, .
\end{equation}
When $B_0$ is defined as \eqref{eq:Bdef}, we have $\gamma = 1-\tau$.
\end{lemma}
The proof of this lemma is deferred to Appendix~\ref{app:proof_reco}. The first bound \eqref{eq:guaranteeP} was proved in~\cite{Lei2015}, we extend it to the Laplacian case.
Note that $\hat L$ could be an estimator of $L(P)$ without being of the form $\hat L = L(M)$ for some matrix $M$.

Using this lemma, in the static SBM case, the goal is to find estimators $\hat P$ or $\hat L$ that concentrates around $P$ or $L(P)$ in spectral norm. In the dynamic case, where the goal is to estimate the communities at a particular time $t$, we seek the best estimators for $P_t$ or $L(P_t)$. As outlined in the previous section, we will consider smoothed versions of the adjacency matrix $\Asmooth_t$, and prove concentration of $\Asmooth_t$ around $P_t$ and $L(\Asmooth_t)$ around $L(P_t)$.

\begin{remark}
Assuming that all community sizes are of the order of $\frac{n}{K}$ and $\tau$ is fixed, the error in the adjacency case \eqref{eq:guaranteeP} scales as $\frac{K^2}{n^2 \alpha_n^2} \norm{\hat P - P}^2$, and in the normalized Laplacian case the error \eqref{eq:guaranteeL} scales as $K^2 \norm{\hat L - L(P)}^2$. Also note that, when $\ntaumax \sim \frac{n}{K}$, then the error \eqref{eq:guaranteeL} is as $ \norm{\hat L - L(P)}^2$. This does not explicitely depend on $\alpha_n$ or $K$, however these quantities will naturally appear in the concentration of the Laplacian.
\end{remark}

The next sections will therefore be devoted in analyzing the spectral concentration rates of the various estimators. Table \ref{tab:rate} summarize our results and compare them with previous works. As we will see in the next section, our main contribution is to weaken the hypothesis on the sparsity $\alpha_n$, and relate it to the regularity of the DSBM $\varepsilon_n$. We also provide the best bound available for the normalized Laplacian in the static case, and the first bound in the dynamic case.

In Figure \ref{fig:n}, we illustrate numerically the spectral concentration of $\Aexp_t$ and $L(\Aexp_t)$, and their actual clustering performance, with respect to the forgetting factor $\lambda$. We see that there is a slight discrepancy between the $\lambda$ that minimizes the spectral bound, and the one that yields the best clustering result. As we will see in the next sections, the $\lambda$ that minimizes the spectral error is theoretically of the order of $\sqrt{\alpha_n n \varepsilon_n}$. This rate is indeed verified numerically for the spectral error, however the actual best clustering performance deviates slightly. This indicates that spectral norm concentration probably does not yield sharp bounds in examining the performance of SC.


 
 \begin{figure}[htbp]
 \centering
  \begin{subfigure}[t]{\figsize\linewidth}
    \centering
    \includegraphics[width=1.0\linewidth]{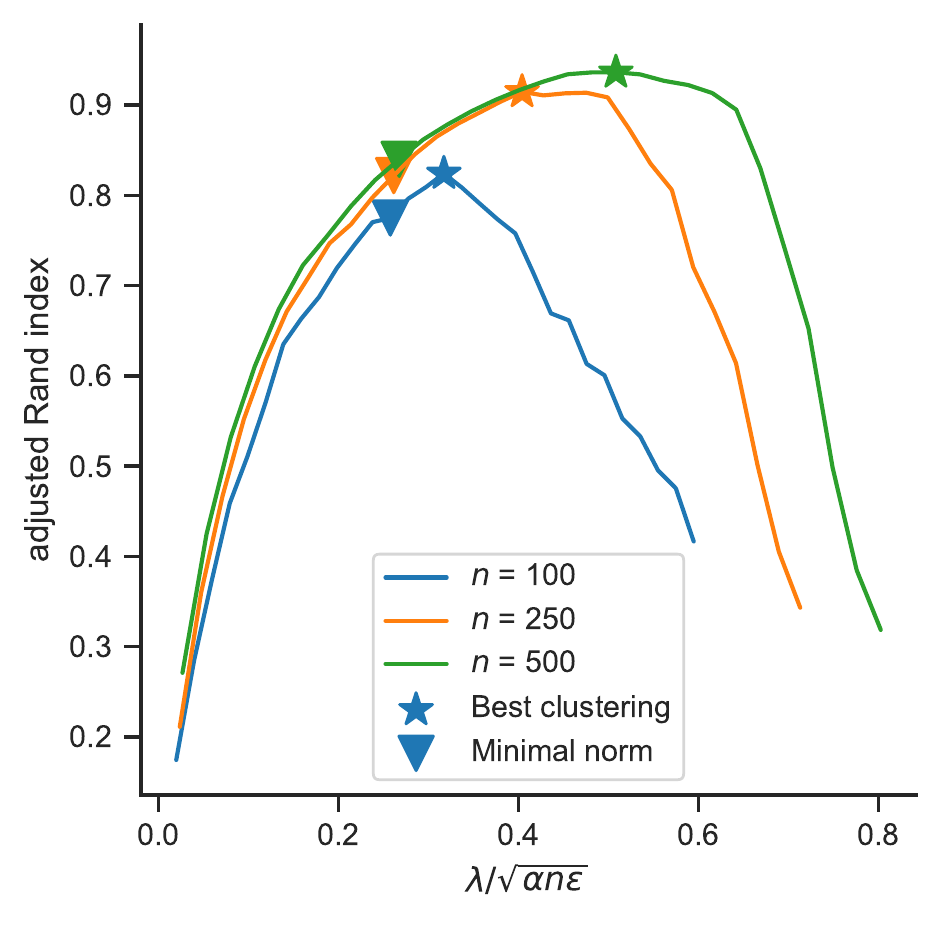} 
    \caption{Adjusted Rand index in function of the normalized forgetting factor $\lambda / \sqrt{\alpha_n n \epsilon_n}$ for the adjacency matrix.} 
    \label{fig:n:accA} 
  \end{subfigure}
  \hspace{1em}
  \begin{subfigure}[t]{\figsize\linewidth}
    \centering
    \includegraphics[width=1.0\linewidth]{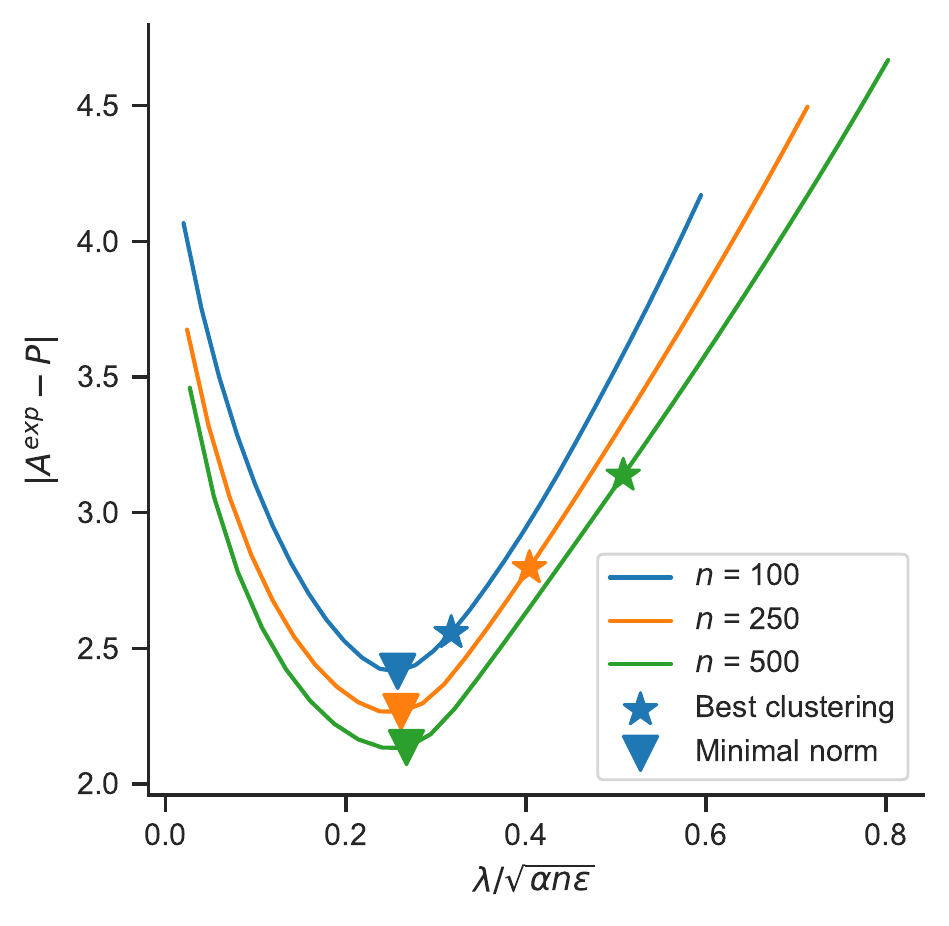} 
    \caption{Approximation (in norm) of $P$ by $\Aexp$ in function of the normalized forgetting factor $\lambda / \sqrt{\alpha_n n \epsilon_n}$ for the adjacency matrix.} 
    \label{fig:n:normA} 
  \end{subfigure}
    \begin{subfigure}[t]{\figsize\linewidth}
    \centering
    \includegraphics[width=1.0\linewidth]{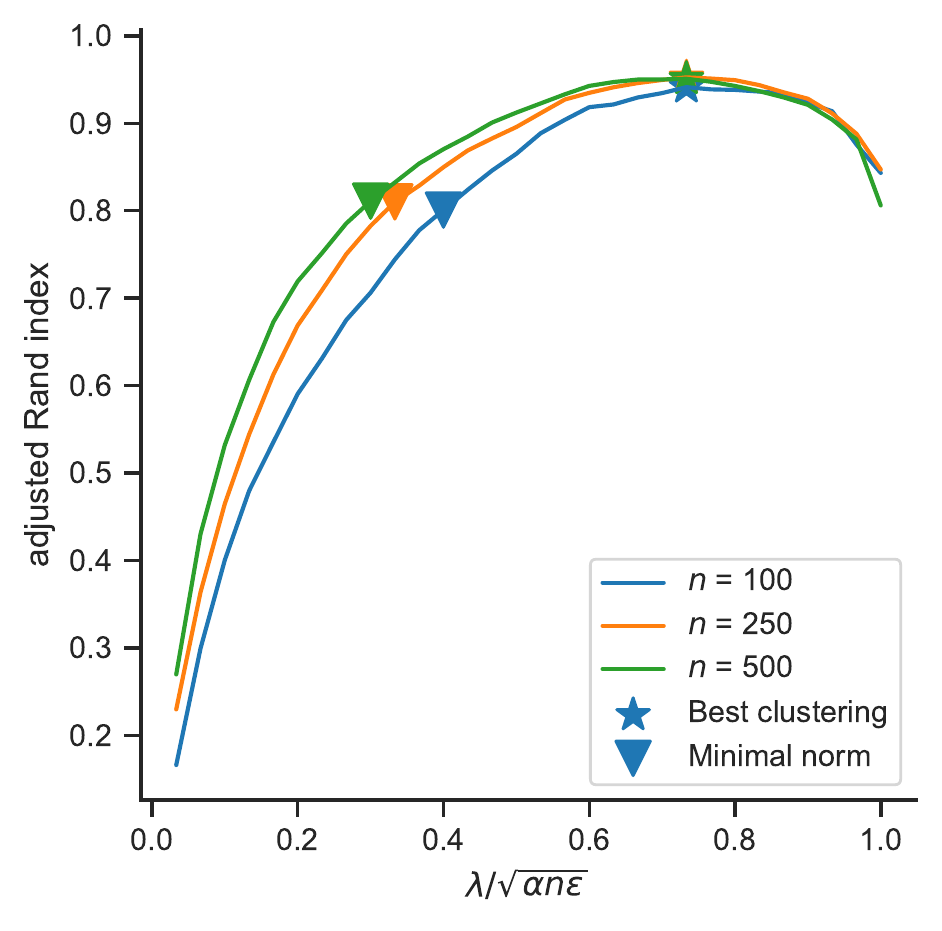} 
    \caption{Adjusted Rand index in function of the normalized forgetting factor $\lambda / \sqrt{\alpha_n n \epsilon_n}$ for the normalized Laplacian matrix.} 
    \label{fig:n:accL} 
  \end{subfigure}
  \hspace{1em}
  \begin{subfigure}[t]{\figsize\linewidth}
    \centering
    \includegraphics[width=1.0\linewidth]{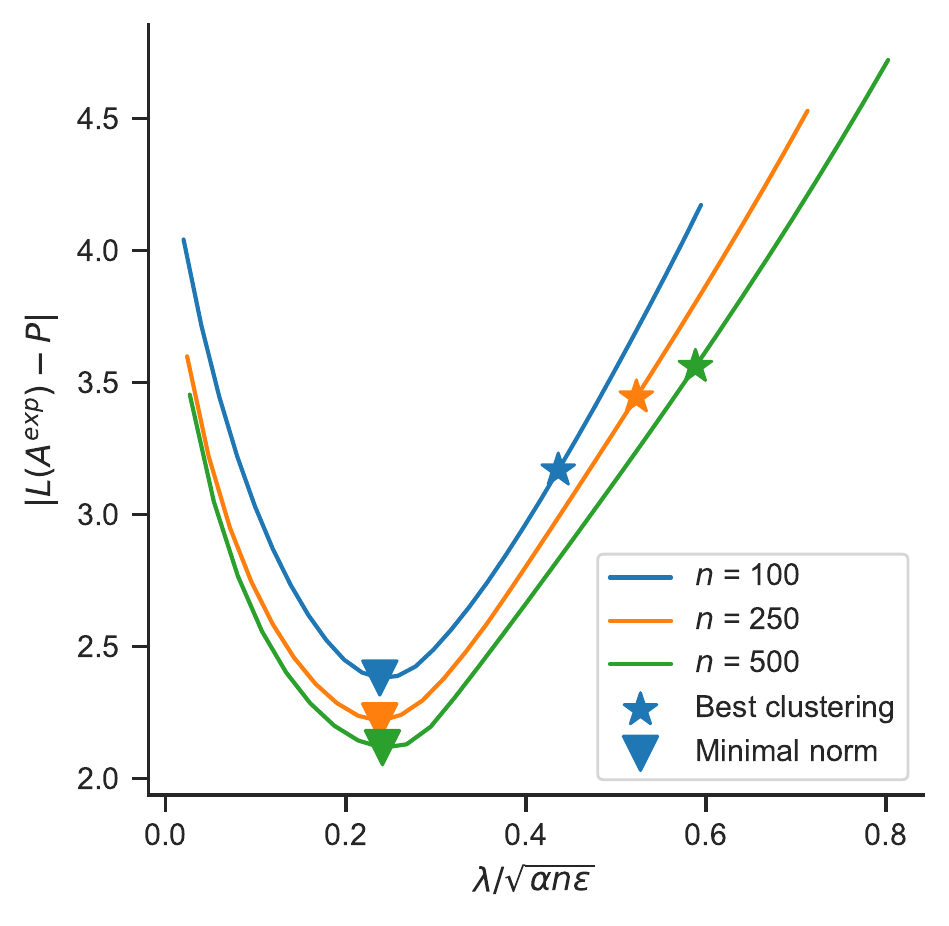} 
    \caption{Approximation (in norm) of $P$ by $L(\Aexp)$ in function of the normalized forgetting factor $\lambda / \sqrt{\alpha_n n \epsilon_n}$ for the normalized Laplacian matrix.} 
    \label{fig:n:normL} 
  \end{subfigure}
  \caption{SC on synthetic data. Comparison between the forgetting factor $\lambda$ that minimizes the spectral error, and the one that yields the best clustering result.}
  \label{fig:n} 
\end{figure}

\begin{table}
\def\arraystretch{1.5}
\centering
\begin{tabular}{|l|l|l|l|l|l|l|}
\hline
  & $A$ static & $L$ static & $A$ dyn. & $L$ dyn. & Hyp. & $E \to 0$ \\ \hline
    \cite{Oliveira2009} & $\log n$ & $\order{1}$ & & & $\alpha_n \gtrsim \frac{\log n}{n}$ & Yes \\
  \cite{Lei2015} & $\sqrt{\alpha_n n}$ & & & & $\alpha_n \gtrsim \frac{\log n}{n}$ & Yes \\
  \cite{Pensky2019} & $\sqrt{\alpha_n n}$ & & $\sqrt{\alpha_n n \rho_n}$ & & $\alpha_n \gtrsim \frac{\log n}{n}$ & Yes \\
  \cite{Bandeira2016a} & $\sqrt{\log n}$ & & & & $\alpha_n \gtrsim \frac{1}{n}$ & No \\
  Us & $\sqrt{\alpha_n n}$ & $\frac{1}{\sqrt{\alpha_n n}}$ & $\sqrt{\alpha_n n \rho_n}$ & $\sqrt{\frac{\rho_n}{\alpha_n n}}$ & $\frac{\alpha_n}{\rho_n} \gtrsim \frac{\log n}{n}$ & Yes \\ \hline
\end{tabular}
\caption{Concentration rates in spectral norm of the adjacency matrix and normalized Laplacian, in the static or dynamic case, with respect to the sparsity parameter $\alpha_n$, and the factor $\rho_n \eqdef \min(1, \sqrt{n\alpha_n \varepsilon_n})$ which includes the regularity $\varepsilon_n$. The last column indicate convergence of the error using Lemma \ref{lem:reco}. This table does not include methods with regularization~\cite{Le2018}.}
\label{tab:rate}
\end{table}

\section{Spectral concentration of the adjacency matrix}\label{sec:adj}

We start by recalling the result of~\cite{Lei2015} in the static case and prove an interesting minor improvement in some cases, then we examine the result for DSBM of~\cite{Pensky2019} and state our main contribution, that is, a weakening of the sparsity hypothesis for this case.

\subsection{Static case}
In their landmark paper~\cite{Lei2015}, Lei and Rinaldo analyze the relatively sparse case $\alpha_n \gtrsim \frac{\log n}{n}$ and show that, with probability at least $1-n^{-\nu}$ for some $\nu>0$, the adjacency matrix concentrates as
\begin{equation}\label{eq:concentrationLei}
\norm{A - P} \lesssim \sqrt{n \alpha_n}
\end{equation}
Therefore, by Lemma \ref{lem:reco}, using $A$ as an estimator for $P$ in an SC algorithm leads to an error $E(\hat \Theta, \Theta) \lesssim \frac{K^2}{\alpha_n n}$, such that $E(\hat \Theta, \Theta) \to 0$ whenever $K = o(\sqrt{n \alpha_n})$.
As a minor contribution, we remark that it is not hard to prove the following Lemma that improves over their result in the particular case when $B_0$ is defined as \eqref{eq:Bdef}.
\begin{proposition}\label{prop:improvedLei}
Consider a static SBM where $B_0$ is defined as \eqref{eq:Bdef}, assume that the community sizes $n_1,\ldots,n_K$ are comprised between $n_{\min}$ and $n_{\max}$, and that
\begin{equation}\label{eq:improvedLeiHyp}
\alpha_n \gtrsim \frac{\log n}{\ntaumin}
\end{equation}
Then, for all $\nu>0$, there exists a constant $C_\nu$ such that, with probability at least $1-\sum_k n_k^{-\nu}$, it holds that
\begin{equation}\label{eq:improvedLei}
\norm{A-P} \leq C_\nu \sqrt{\ntaumax \alpha_n}
\end{equation}
\end{proposition}

\begin{proof}
Denote by $S_1,\ldots,S_K \subset [n]$ the subset of indices of each community, assume without lost of generality that the nodes are ordered such that the $S_k$ are consecutive in $[n]$, that is, $S_1=\{1,\ldots,n_1\}$, $S_2=\{n_1+1,\ldots, n_1+n_2\}$, and so on. Define $A_k = A_{S_k, S_k} \in \set{0,1}^{n_k \times n_k}$ the adjacency matrix of the subgraph of nodes from the $k$th community. Note that by our assumption on $B$ we have $P_k = P_{S_k, S_k} = \alpha_n 1_{n_k} 1_{n_k}^\top$. Denote $A'\in\set{0,1}^{n \times n}$ the block matrix containing the $A_k$ on its diagonal of blocks, similarly $P'$, and $A'' = A-A'$, $P'' = P-P'$. We have
\[
\norm{A-P} \leq \norm{A'-P'} + \norm{A''-P''} = \max_k \norm{A_k - P_k} +\norm{A''-P''}
\]
where the equality is valid because $A'-P'$ is a block diagonal matrix. From Lei and Rinaldo's result above, for each $k$, if $\alpha_n \gtrsim \frac{\log n_k}{n_k}$, then with probability at least $1-n_k^{-\nu}$ it holds that $\norm{A_k - P_k} \lesssim \sqrt{n_k \alpha_n}$, such that $\norm{A'-P'} \lesssim \sqrt{n_{\max}\alpha_n}$. For the second term, we note that $A''$ is an adjacency matrix generated by the SBM corresponding to $P''$, whose maximal probability is $\tau\alpha_n$. Hence, if $\tau \alpha_n \gtrsim \frac{\log n}{n}$, then with probability $1-n^{-\nu}$ we have $\norm{A''-P''} \lesssim \sqrt{\tau n\alpha_n}$. We conclude with a union bound.
\end{proof}

This Lemma provides a better error rate than~\cite{Lei2015} when $\tau$ goes to $0$ with $K$, at the price of requiring a higher $\alpha_n$. For instance, when the communities sizes are balanced $n_k \sim \frac{n}{K}$, and we have $\tau \sim \frac{1}{K}$ and $\alpha_n \sim \frac{K\log n}{n}$, Lei and Rinaldo's rate \eqref{eq:concentrationLei} yields $E(\hat \Theta, \Theta) \lesssim \frac{K}{\log n}$ and converge only for $K=o(\log n)$, while using Proposition \ref{prop:improvedLei} we get $E(\hat \Theta, \Theta) \lesssim \frac{1}{\log n}$. The latter does not depend on $K$, which may grow with any rate in the number of nodes (recalling that $\tau$ and $\alpha_n$ depend on $K$, and that there must be a $\nu>0$ such that $K^{\nu +1} n^{-\nu}\to 0$ to obtain a probability rate that goes to $1$).

\subsection{Dynamic case}

In~\cite{Pensky2019}, Pensky and Zhang analyze the dynamic case with Lei and Rinaldo's proof technique. They consider the deterministic DSBM model in the almost sparse case $\alpha_n\gtrsim \frac{\log n}{n}$ and the uniform estimator \eqref{eq:estimatorUnif}. Defining a factor
\begin{equation}\label{eq:rhodefPZ}
\rhoPZ_n = \min(1, \sqrt{n\alpha_n \varepsilon_n})\, ,
\end{equation}
they show that, for an optimal choice of window size $r \sim \frac{1}{\rhoPZ_n}$, it holds that
\begin{equation}\label{eq:penskybound}
\norm{\Aunif_t - P_t} \lesssim \sqrt{n \alpha_n\rhoPZ_n}
\end{equation}
In particular, the concentration is better if $\rhoPZ_n = o(1)$, that is:
\begin{equation}\label{eq:epsilon_penskycondition}
\varepsilon_n = o\pa{\frac{1}{\alpha_n n}}\, .
\end{equation}
In other words, there is an improvement if we assume sufficient \emph{smoothness} in time, which then leads to a better error rate $E(\hat \Theta, \Theta) \lesssim \frac{K^2 \rhoPZ_n}{\alpha_n n}$ when using $\Aunif_t$ in the SC algorithm. Note that, with this proof technique a constant smoothness $\varepsilon_n \sim 1$ does not improve the error rate (see remark \ref{rem:constant_eps}).

We remark that, despite the assumption on the smoothness and the availability of more data, the result above still assumes the relative sparse case. However, with sufficient smoothness, it should be possible to weaken the hypothesis made on the sparsity $\alpha_n$, since intuitively, if there is more data available where the communities are almost the same as the present time step, the density of edges should not need to be as large. We solve this in the following theorem, which is the central contribution of this paper.
\begin{theorem}\label{thm:main_adj}
Consider the deterministic DSBM with any $B_0$. Define
\begin{equation}\label{eq:rhodef}
\rho_n \eqdef \min\pa{1, \sqrt{\ntaumax \alpha_n \varepsilon_n}}
\end{equation}
Assume $t \geq t_{\min}\eqdef \frac{\log\pa{\frac{\rho_n}{\alpha_n n}}}{2\log(1-\rho_n)}$, and
\begin{equation}
\label{eq:cond_alpha_main}
\frac{\alpha_n}{\rho_n} \gtrsim \frac{\log n}{n}\, .
\end{equation}
Consider either the uniform estimator $\Asmooth_t = \Aunif_t$ with $r \sim \frac{1}{\rho_n}$ or the exponential estimator $\Asmooth_t = \Aexp_t$ with $\lambda \sim \rho_n$.

For all $\nu>0$, there is a universal constant $C_\nu$ such that, with probability at least $1-n^{-\nu}$, it holds that
\begin{equation}\label{eq:main_adj}
\norm{\Asmooth_t - P_t} \leq C_\nu \sqrt{n\alpha_n \rho_n} .
\end{equation}
\end{theorem}

In this theorem, we improve over \cite{Pensky2019} in several ways. First, we improve $\rhoPZ_n$ to $\rho_n$ by replacing $n$ with $\ntaumax\leq n$. In the case where $\sum_\ell (B_0)_{k\ell}$ stays bounded, for instance if it is defined as \eqref{eq:Bdef} with $\tau \sim \frac{1}{K}$, we have $\ntaumax \sim \frac{n}{K}$ and this improves the bound \eqref{eq:main_adj} compared to \eqref{eq:penskybound}. We also extend the result to the exponential estimator with the right choice of forgetting factor.

More importantly, the main feature of our result is the weaker condition \eqref{eq:cond_alpha_main}, which relates the sparsity and the smoothness of the DSBM. Strinkingly, if
\begin{equation}\label{eq:epsilon_ourcondition}
\varepsilon_n \sim \frac{n/\ntaumax}{\log^2 n}\, ,
\end{equation}
which is a slight strengthening of \eqref{eq:epsilon_penskycondition}, then our result is valid in the sparse regime $\alpha_n \sim \frac{1}{n}$, which is a significant improvement compared to previous works. In any case, if we have exactly $\frac{\alpha_n}{\rho_n} \sim \frac{\log n}{n}$, then as previously Lemma \ref{lem:reco} yields that $E(\hat \Theta, \Theta) \to 0$ when $K = o(\sqrt{\log n})$.

\begin{proposition}\label{prop:markov_adj}
The result of Theorem \ref{thm:main_adj} stays valid under the Markov DSBM, by replacing $\ntaumax$ with $n$ everywhere, and assuming
\begin{equation}\label{eq:cond_epsilon}
\epsilon_n \gtrsim \sqrt{\frac{\log n}{n}}
\end{equation}
(in this case, ``with probability at least $1-n^{-\nu}$'' refers to joint probability on both the $A_t$ and the $\Theta_t$).
\end{proposition}
The above Lemma shows that the Markov DSBM yields the exact same error bounds than the deterministic DSBM model, but since we do not assume a maximal community size here, $\ntaumax$ is replaced with $n$. Furthermore, $\varepsilon_n$ cannot be too small to still obtain a polynomial probability of failure. Nevertheless, the condition \eqref{eq:cond_epsilon} is much weaker than the rate \eqref{eq:epsilon_ourcondition} for instance, such that the sparse regime with sufficient smoothness if still valid.

\begin{remark}\label{rem:constant_eps}
As already observed in~\cite{Pensky2019}, with this proof technique, a constant $\varepsilon_n$, or in other words, a fraction of changing nodes $s$ that grows linearly with $n$, does not result in an improvement of the rate of the error bounds compared to the static case. Following the statistical physic approach in the sparse static case \cite{Krzakala2013, Mossel2017, Abbe2018}, a conjecture on the detectability threshold in the sparse case and $\varepsilon_n\sim 1$ has been formulated in \cite{Ghasemian2016}, but the proof is still open. Note that, as mentioned before, even in the static case, this analysis does not cover classic SC algorithm, or the case $K>2$.
\end{remark}

\section{Spectral concentration of the normalized Laplacian}\label{sec:Lapl}

As mentioned in the introduction, the spectral concentration of the normalized Laplacian has been less studied than the adjacency matrix, even in the static case. Many works study the asymptotic spectral convergence of the normalized Laplacian in the dense case \cite{VonLuxburg2008}, but few examine non-asymptotic bounds.

\subsection{Static case} Among the few existing bounds, \cite{Oliveira2009} proves a concentration in $\order{1}$ in the relatively sparse case, and~\cite{Rohe2011} proves a concentation in Frobenius norm but with the stronger condition $\alpha_n \gtrsim \frac{1}{\sqrt{\log n}}$. An important corollary of our study of the dynamic case is to significantly improves over these results, and obtain, to our knowledge, the best bound available in the relatively sparse case. We state the following proposition for any Bernoulli matrix (not necessarily SBM).

\begin{proposition}[Normalized Laplacian, static case]\label{prop:static_Lapl}
Let $A$ be a symmetric matrix with independent entries $a_{ij} \sim Ber(p_{ij})$. Assume $p_{ij} \leq \alpha_n$, and that there is $\ntaumin,\ntaumax$ such that for all $i$, $\alpha_n \ntaumin \leq \sum_j p_{ij}\leq \alpha_n \ntaumax$, and $\commratio = \frac{\ntaumax}{\ntaumin}$. For all $\nu >0$, there are constants $C_\nu, C'_\nu$ such that: if
\begin{equation}\label{eq:cond_alpha_static_Lapl}
\alpha_n \geq C'_\nu \commratio \frac{\log n}{\ntaumin}
\end{equation}
then with probability at least $1-n^{-\nu}$ we have
\[
\norm{L(A) - L(P)} \lesssim \frac{C_\nu \commratio \sqrt{n}}{\ntaumin \sqrt{\alpha_n}}
\]
\end{proposition}
\begin{proof}
This is a direct consequence of Theorem \ref{thm:concentration_Lapl} in Section \ref{sec:proof}.
\end{proof}
In other words, when $\ntaumin \sim n$ (for instance when all the $p_{ij} /\alpha_n$ are bounded below), then in the relatively sparse case the spectral concentration of the normalized Laplacian is in $\frac{1}{\sqrt{\log n}}$, which is a strict improvement over existing bounds.

Let us comment a bit on the condition \eqref{eq:cond_alpha_static_Lapl}. When $\ntaumin = o(n)$ or $\commratio^{-1} = o(1)$, it is stronger than the relatively sparse case. The attentive reader would also remark the subtle interplay of the quantifiers with the rate $\nu$: in the analysis of the adjacency matrix in the previous section, any multiplicative constant between $\alpha_n$ and $\frac{\log n}{n}$ was acceptable, and the rate $\nu$ only forced a multiplicative constant $C_\nu$ in the final error bound. Here, the rate $\nu$ also imposes a multiplicative constant $C'_\nu$ in the sparsity hypothesis.

\subsection{Dynamic case} 

To our knowledge, the normalized Laplacian in the DSBM has never been studied theoretically. Our result is the following.

\begin{theorem}\label{thm:main_Lapl}
Consider the deterministic DSBM with $B$ satisfying \eqref{eq:Bdef}, and either the uniform estimator $\Asmooth_t = \Aunif_t$ with $r \sim \frac{1}{\rho_n}$ or the exponential estimator $\Asmooth_t = \Aexp_t$ with $\lambda = \rho_n$. Assume $t\geq t_{\min}$.

For all $\nu>0$, there exist universal constants $C_\nu, C'_\nu >0$ such that: if
\begin{equation}\label{eq:cond_alpha_Lapl}
\frac{\alpha_n}{\rho_n} \geq C'_\nu \commratio \frac{\log n}{\ntaumin}
\end{equation}
then with probability at least $1-n^{-\nu}$, it holds that
\begin{equation}\label{eq:bound_Lapl}
\norm{L(\Asmooth_t) - L(P_t)} \leq C_\nu \commratio \sqrt{\frac{n \rho_n}{\ntaumin^2 \alpha_n}}\, .
\end{equation}

\end{theorem}

In the case of balanced communities, the result of Theorem \ref{thm:main_Lapl} combined with Lemma \ref{lem:reco} yields the same error rate than in the case of the adjacency matrix with Theorem \ref{thm:main_adj} and Lemma \ref{lem:reco}, even in terms of $K$ when $\ntaumin, \ntaumax \sim \frac{n}{K}$. Note however that in the latter, the condition \eqref{eq:cond_alpha_Lapl} is slightly stronger than \eqref{eq:cond_alpha_main}, similar to the static case. In practice however, it is well-known that the normalized Laplacian generally performs better (Fig. \ref{fig:basic}).

\section{Proofs}\label{sec:proof}

In this section, we provide the proof of our main results, largely inspired by \cite{Lei2015} and \cite{Pensky2019}. The technical computations are given in appendix. Despite some similarity with \cite{Lei2015} and \cite{Pensky2019}, we strove to make the proofs self-contained.

\subsection{Preliminaries}

We place ourselves at a particular time $t$. Both estimators $\Aunif_t$ and $\Aexp_t$ can be written as a weighted sum
\begin{equation}\label{eq:estimatorSmooth}
\Asmooth_t = \sum_{k=0}^t \beta_k A_{t-k}\, ,
\end{equation}
where $\beta_0 = \ldots = \beta_{r-1} = \frac{1}{r}$ and $\beta_k = 0$ for $k\geq r$ in the uniform case, and $\beta_k = \lambda (1-\lambda)^{k}$ for $k<t$ and $\beta_t=(1-\lambda)^t$ in the exponential case.
As we will see, our results will be valid for any estimator of the form \eqref{eq:estimatorSmooth}, with weights $\beta_k \geq 0$ that satisfy: there are constants $\betamax, \Cbeta, \CCbeta >0$ such that:
\begin{align}
\label{eq:cond_beta}
&\sum_{k=0}^t \beta_k = 1,\quad \beta_k \leq \betamax, \quad \sum_{k=0}^t \beta_k^2 \leq \Cbeta\betamax,\\
& \sum_{k=0}^t \beta_k \min(1,\sqrt{k\epsilon_n}) \leq \CCbeta \sqrt{\frac{\epsilon_n}{\betamax}}\notag
\end{align}
In words, the weights must naturally sum to $1$ and be bounded; the sum of their squares must be small; and they must decrease faster than $\sqrt{k}$, which is roughly the rate at which the past communities $\Theta_{t-k}$ deviate from $\Theta_t$. It is not difficult to show that the uniform and exponential estimator satisfy these conditions.

\begin{lemma}
The weights in the uniform estimator \eqref{eq:estimatorUnif} satisfy \eqref{eq:cond_beta} with $\betamax = \frac{1}{r}$, $\Cbeta = \CCbeta = 1$. If $t \geq t_{\min} = \frac{\min(\log(\epsilon_n/\betamax), \log \betamax)}{2\log(1-\betamax)}$, the weights in the exponential estimator \eqref{eq:estimatorExp} satisfy \eqref{eq:cond_beta} with $\betamax = \lambda$, $\Cbeta = \frac{3}{2}$, $\CCbeta = 2$.
\end{lemma}
\begin{proof}
The computations are trivial in the uniform case, where the last condition is implied by the stronger property $\sum_k \beta_k \sqrt{k} \leq \sqrt{r} = \sqrt{\frac{1}{\betamax}}$. In the exponential case, we have $\betamax = \lambda$, and
\begin{align*}
\sum_k \beta_k^2 &= \lambda^2 \sum_{k=0}^{t-1} (1-\lambda)^{2k} + (1-\lambda)^{2t} \leq \lambda^2 \sum_{k=0}^{\infty} (1-\lambda)^{2k} + \lambda \leq \frac{3}{2} \lambda
\end{align*}
where the first inequality is valid since $t \geq \frac{\log \betamax}{2\log(1-\betamax)}$, and thus $\Cbeta = \frac{3}{2}$. Next, we have
\begin{align*}
\sum_{k=0}^t \beta_k \min(1,\sqrt{k\epsilon_n}) \leq \sqrt{\epsilon_n} \lambda \sum_{k=0}^\infty \sqrt{k}(1-\lambda)^k + (1-\lambda)^t
\end{align*}
Since $t \geq \frac{\log(\epsilon_n/\betamax)}{2\log(1-\betamax)}$, we get $(1-\lambda)^t \leq \sqrt{\frac{\epsilon_n}{\lambda}}$ and
\begin{align*}
\sum_{k=0}^{\infty} (1-\lambda)^k \sqrt{k} &\leq \sqrt{\sum_{k=0}^{\infty} (1-\lambda)^k}\sqrt{\sum_{k=0}^{\infty} (1-\lambda)^k k} = \sqrt{\frac{1}{\lambda}}\sqrt{\frac{1-\lambda}{\lambda^2}} \leq \frac{1}{\lambda^{3/2}}
\end{align*}
and therefore we obtain the desired inequality with $\CCbeta = 2$.
\end{proof}

\subsection{Concentration of adjacency matrix: proof of Theorem~\ref{thm:main_adj}}\label{sec:proof-adj}

For an estimator of the form \eqref{eq:estimatorSmooth}, our goal is to bound $\norm{\Asmooth_t-P_t}$. We define $\Psmooth_t \eqdef \sum_{k=0}^t \beta_k P_{t-k}$, and divide the error in two terms:
\begin{equation}\label{eq:decompose_error}
\norm{\Asmooth_t - P_t} \leq \norm{\Asmooth_t - \Psmooth_t} + \norm{\Psmooth_t - P_t}
\end{equation}
The first error term corresponds to the difference between $\Asmooth_t$ and its expectation (up to the diagonal terms). Intuitively, it \emph{decreases} when the amount of smoothing \emph{increases}, that is, when $r$ increases or $\lambda$ gets close to $0$, since the sum of matrices is taken over more values. The second term is the difference between the smoothed matrix of probability connection and its value at time $t$. This time, it will increase when the amount of smoothing increases, since the past communities will be increasingly present in $\Psmooth_t$. Once we have the two bounds, we can balance them to obtain an optimal value for $r$ or $\lambda$, respectively $\frac{1}{\rho_n}$ and $\rho_n$.

\subsubsection{Bound on the first term} The first bound will be handled by the following general concentration theorem. This is where we are able to weaken the hypothesis on the sparsity.

\begin{theorem}\label{thm:concentration}
Let $A_1,\ldots,A_t \in \{0,1\}^{n\times n}$ be $t$ symmetric Bernoulli matrices whose elements $a_{ij}^{(k)}$ are independent random variables:
\[
a_{ij}^{(k)} \sim \text{Ber}(p_{ij}^{(k)}),\quad a_{ji}^{(k)} = a_{ij}^{(k)},\quad a_{ii}^{(k)} = 0
\]
Assume $p_{ij}^{(k)} \leq \alpha_n$. Consider non-negative weights $\beta_k$ that satisfy \eqref{eq:cond_beta}.
Denoting $A = \sum_{k=0}^t \beta_k A_{t-k}$ and $P = \mathbb{E}(A)$, there is a universal constant $C$ such that for all $c > 0$ we have 
\begin{align}
\mathbb{P}\left(\norm{A-P} \geq C(1+c)\sqrt{n \alpha_n \betamax} \right) \leq&~ e^{-\left(\frac{c^2/2}{2\Cbeta + \frac{2}{3} c}- \log(14)\right) n} \label{eq:concentration}\\
&+ e^{ - \frac{c^2/2}{\Cbeta +2c/3} \cdot \frac{n \alpha_n}{\betamax}+ \log n } + n^{-\frac{c}{4}+6} \notag
\end{align}
\end{theorem}
This theorem is proved in Appendix~\ref{app:proof_thm_concentration}. Its proof is heavily inspired by~\cite{Lei2015} and~\cite{Pensky2019}: the spectral norm is expressed as a maximization problem over the sphere, and for each point of the sphere the obtained sum is divided into so-called ``light'' terms, for which Berstein's concentration inequality is sufficient, and more problematic ``heavy'' terms, that require a complex concentration method. We obtain our weaker sparsity hypothesis in a small but crucial part of this second step, the so-called \emph{bounded degree lemma}.

\begin{lemma}[Bounded degree.]\label{lem:bounded-degree}
Denote $d_{i,t} = \sum_j (A_t)_{ij}$ the degree of node $i$ at time $t$, $d_i = \sum_{k=0}^t \beta_k d_{i,t-k}$ the smoothed degree and $\bar d_i = \mathbb{E} d_i$. Then, for all $c$,
\[
\mathbb{P}(\max_i \abs{d_i - \bar d_i} \geq c n \alpha_n) \leq \exp\left( - \frac{c^2/2}{\Cbeta +2c/3} \cdot \frac{n \alpha_n}{\betamax}+ \log n\right)
\]
\end{lemma}

\begin{proof}
We use Bernstein's inequality. For any fixed $i$ we have
\[
d_i = \sum_{k=0}^t \sum_{j\neq i} \beta_k a_{ij}^{(t-k)} = \sum_{k,j} Y_{jk}
\]
where $Y_{jk} = \beta_k a_{ij}^{(t-k)}$ are such that $\mathbb{E}(Y_{jk}) = \beta_k p_{ij}^{(t-k)} \leq \beta_k \alpha_n$, $\abs{Y_{jk} - \mathbb{E}Y_{jk}} \leq (\alpha_n +1) \beta_k \leq 2\betamax$, and $Var(Y_{jk}) \leq \beta_k^2 \alpha_n$ such that $\sum_{k,j} Var(Y_{jk}) \leq \Cbeta n\alpha_n\betamax$.

Therefore, applying Berstein's inequality, we have
\begin{align*}
\mathbb{P}(\abs{d_i-\bar d_i} \geq c n \alpha_n) &\leq \exp\left(-\frac{c^2n^2 \alpha_n^2/2}{\Cbeta n\alpha_n\betamax + \frac{2}{3}\betamax cn \alpha_n}\right)
\end{align*}
Applying a union bound over the nodes $i$ proves the result.
\end{proof}
In the static case \cite{Lei2015} where $\betamax = 1$, the bounded degree lemma is exactly where the relative sparsity hypothesis $\alpha_n \gtrsim \frac{\log n}{n}$ is needed, otherwise the probability of failure diverges. In the dynamic case, we see that $\betamax$ (which we will ultimately set at $\rho_n$) intervenes and gives our final hypothesis on sparsity and smoothness.

Applying Theorem \ref{thm:concentration}, we obtain that for any fixed $\Theta_0,\ldots,\Theta_t$, if $\frac{n\alpha_n}{\betamax} \gtrsim \log n$, then for any $\nu>0$ there is a constant $C_\nu$ such that with probability at least $1-n^{-\nu}$
\begin{align*}
\norm{\Asmooth_t - \Psmooth_t} &\leq \norm{\Asmooth_t - \mathbb{E}(\Asmooth_t)} + \norm{\text{diag}(\Psmooth_t)} \\
&\leq C_\nu \sqrt{n\alpha_n \betamax} + \alpha_n
\end{align*}
Since in all considered cases we will have $\betamax \geq 1/n$ the second term is negligible, and we obtain
\begin{equation}\label{eq:error_first}
\norm{\Asmooth_t - \Psmooth_t} \lesssim \sqrt{n\alpha_n \betamax}
\end{equation}

\subsubsection{Second term} The second error term in \eqref{eq:decompose_error} is handled slightly differently in the deterministic and Markov DSBM, even if the final bound is the same.

\begin{lemma}\label{lem:error_second_deter}
Consider the deterministic DSBM, with weights that satisfy \eqref{eq:cond_beta}. It holds that
\begin{equation}
\label{eq:error_second}
\norm{\Psmooth_t - P_t} \lesssim \CCbeta \alpha_n \sqrt{\frac{n \ntaumax \epsilon_n}{\betamax}}
\end{equation}
\end{lemma}
\begin{proof}
Since the weights sum to $1$, we decompose
\[
\norm{\Psmooth_t - P_t} \leq \sum_k \beta_k \norm{P_{t-k} - P_t} \leq \sum_k \beta_k \norm{P_{t-k} - P_t}_F
\]
where $\norm{\cdot}_F$ is the Frobenius norm. Consider $P= \Theta B \Theta^\top$ and $P'=\Theta' B (\Theta')^\top$ two probability matrices such that there is a set $\mathcal{S}$ of nodes that have changed communities. We have then:
\begin{align*}
\norm{P-P'}_F^2 &= \sum_{i \in \mathcal{S}} \sum_j (p_{ij} - p'_{ij})^2 + (p_{ji} - p'_{ji})^2 \leq 4 \sum_{i\in\mathcal{S}}\sum_j p_{ij}^2 + (p'_{ij})^2 \\
&\leq 8 \alpha_n^2\abs{\mathcal{S}} n_{\max} \max_k \sum_\ell (B_0)_{k\ell}^2 \leq 8 \abs{\mathcal{S}} \alpha_n^2 \ntaumax
\end{align*}
Since at most $ks$ nodes have changed community between $P_t$ and $P_{t-k}$, with a maximum of $n$ nodes, we have
\begin{equation}\label{eq:error_second_proof}
\norm{P_{t-k} - P_t}_F^2 \leq 2\alpha_n^2 \ntaumax \min(n, ks) = 2\alpha_n^2 n \ntaumax \min(1, k\epsilon_n)
\end{equation}
Using the hypothesis that we have made on $\sum_k \beta_k \min(1,\sqrt{k\epsilon_n})$, we obtained the desired bound.
\end{proof}

At the end of the day, combining \eqref{eq:decompose_error}, \eqref{eq:error_first} and \eqref{eq:error_second} for both deterministic and Markov DSBM model we obtain with the desired probability:
\begin{equation}
\norm{\Asmooth_t - P_t} \lesssim \err_1(\betamax) + \err_2(\betamax) ~ \text{where} ~ \begin{cases}
\err_1(\beta) \eqdef \sqrt{n\alpha_n \beta} \\
\err_2(\beta) \eqdef \alpha_n \sqrt{\frac{n \ntaumax \epsilon_n}{\beta}}
\end{cases}
\end{equation}
As expected, $\err_1$ decreases and $\err_2$ increases when $\betamax$ decreases. A simple function study show that the sum of the errors is minimized for $\betamax = \rho_n$, which concludes the proof of Theorem \ref{thm:main_adj}.

\subsubsection{Markov DSBM}
Since the bound on the first term \eqref{eq:error_first} is valid for any $\Theta_k$, and the $A_k$ are conditionally independent given the $\Theta_k$, by the law of total probability it is also valid with joint probability at least $1-n^{-\nu}$ on both the $A_k$ and $\Theta_k$ in the Markov DSBM model.
For the bound on the second term, we show that \eqref{eq:error_second_proof} is still valid with high probability, replacing $\ntaumax$ with $n$.
\begin{lemma}\label{lem:error_second_markov}
Consider the Markov DSBM model. We have
\[
\mathbb{P}\pa{ \exists k,~ \norm{P_{t-k} - P_t}_F^2 \geq (8+C)\alpha_n^2 n^2 \min(1, k\epsilon_n)} \leq e^{-2C^2\epsilon_n^2 n + \log\frac{1}{\epsilon_n}}
\]
\end{lemma}
The proof is in Appendix \ref{app:additional_proofs}. 
Using this Lemma, if \eqref{eq:cond_epsilon} is satisfied we obtain that with probability at least $1-n^{-\nu}$, \eqref{eq:error_second_proof} is satisfied for all $k$. Using the rest of the proof of Lemma \ref{lem:error_second_deter}, \eqref{eq:error_second} is valid in the Markov DSBM model, with $n$ instead of $\ntaumax$. The rest of the proof is the same as the deterministic case.

\subsection{Concentration of Laplacian: proof of Theorem~\ref{thm:main_Lapl}}

A crucial part of handling the normalized Laplacian is to lower-bound the degrees of the nodes, since we later manipulate the inverse of the degree matrix. Under our hypotheses, the minimal expected degree is of the order of $\alpha_n \ntaumin$, so we need to bound the deviation of the degrees with respect to this quantity. We revisit the bounded degree lemma.

\begin{lemma}[Bounded degree revisited.]\label{lem:bounded-degree-Lapl}
Under the deterministic DSBM, for all $c$,
\[
\mathbb{P}(\max_i \abs{d_i - \bar d_i} \geq c \ntaumin \alpha_n) \leq \exp\left( - \frac{c^2/2}{\Cbeta +2c/3} \cdot \frac{\ntaumin \alpha_n}{\commratio\betamax}+ \log n\right)
\]
\end{lemma}

\begin{proof}
We do the exact same proof as Lemma \ref{lem:bounded-degree}, but we remark that $\sum_{k,j} Var(Y_{jk}) \leq \Cbeta \ntaumax \alpha_n\betamax$, since $\sum_i p^{(t-k)}_{ij} \leq \alpha_n \ntaumax$ for all $k,i$.
Therefore, applying Berstein's inequality, we have
\begin{align*}
\mathbb{P}(\abs{d_i-\bar d_i} \geq c \ntaumin \alpha_n) &\leq \exp\left(-\frac{c^2\ntaumin^2 \alpha_n^2/2}{\Cbeta \ntaumax\alpha_n\betamax + \frac{2}{3}\betamax c\ntaumin \alpha_n}\right)
\end{align*}
Applying a union bound over the nodes $i$ proves the result.
\end{proof}
To lower-bound $d_i$, we use Lemma \ref{lem:bounded-degree-Lapl} with $0<c<1$, for instance $c= \frac12$. The sparsity hypothesis \eqref{eq:cond_alpha_Lapl} in the theorem comes directly from this: it uses $\ntaumin$ instead of $n$, and the multiplicative constant $C'_\nu$ actually depends on the desired concentration rate $\nu$, unlike the previous case of the adjacency matrix where $\nu$ could be obtained by adjusting $c$ in Lemma \ref{lem:bounded-degree}. Let us now turn to the proof of the theorem.

As before, we divide the bound in two parts:
\begin{equation}\label{eq:decompose_error_Lapl}
\norm{L(\Asmooth_t) - L(P_t)} \leq \norm{L(\Asmooth_t) - L(\Psmooth_t)} + \norm{L(\Psmooth_t) - L(P_t)}
\end{equation}

The first bound is handled with a general concentration theorem.

\begin{theorem}\label{thm:concentration_Lapl}
Let $A_1,\ldots,A_t \in \{0,1\}^{n\times n}$ be $t$ symmetric Bernoulli matrices whose elements $a_{ij}^{(k)}$ are independent random variables:
\[
a_{ij}^{(k)} \sim \text{Ber}(p_{ij}^{(k)}),\quad a_{ji}^{(k)} = a_{ij}^{(k)},\quad a_{ii}^{(k)} = 0
\]
Consider non-negative weights $\beta_k$ that satisfy \eqref{eq:cond_beta}. Denoting $A = \sum_{k=0}^t \beta_k A_{t-k}$ and $P = \mathbb{E}(A)$. Assume $p_{ij}^{(k)} \leq \alpha_n$, and that there is $\ntaumin, \ntaumax$ such that for all $i$, $\alpha_n \ntaumin \leq \sum_{j} p_{ij} \leq \alpha_n \ntaumax$.
Then there is a universal constant $C$ such that for all $c > 0$ we have 
\begin{align}
\mathbb{P}\Big(\norm{L(A)-L(P)} \geq \frac{C(1+c)\commratio}{\ntaumin}&\sqrt{\frac{n\betamax}{\alpha_n}} \Big) \label{eq:concentration_Lapl} \\
\leq&~ e^{-\left(\frac{c^2/2}{2\Cbeta + \frac{2}{3} c}- \log(14)\right) n} + e^{ - \frac{c^2/2}{\Cbeta +2c/3} \cdot \frac{n \alpha_n}{\betamax}+ \log n } \notag \\
&+ e^{ - \frac{1/8}{\Cbeta +1/3} \cdot \frac{\ntaumin \alpha_n}{\commratio\betamax}+ \log n } + n^{-\frac{c}{4}+6} \notag
\end{align}
\end{theorem}
The proof is in Appendix \ref{app:Lapl}. 
Similar to the adjacency matrix case, we thus obtain
\begin{align*}
\norm{L(\Asmooth_t) - L(\Psmooth_t)} \leq \norm{L(\Asmooth_t) - L(\mathbb{E}(\Asmooth_t))}  + \norm{L(\mathbb{E}(\Asmooth_t)) - L(\Psmooth_t)} 
\end{align*}
and by Lemma \ref{lem:technicalLaplacianBound}, the second term is negligible since $\mathbb{E}(\Asmooth_t)$ and $\Psmooth_t$ only differ by their diagonal, of the order of $\alpha_n$.

The second bound is handled in the same way as the adjacency matrix in the deterministic case.

\begin{lemma}\label{lem:error_second_Lapl}
Under the deterministic DSBM, we have
\[
\norm{L(P_t) - L(\Psmooth_t)} \lesssim \frac{C'_\beta \commratio}{\ntaumin} \sqrt{\frac{n \ntaumax \varepsilon_n}{\betamax}} 
\]
\end{lemma}
The proof is in Appendix \ref{app:additional_proofs}.

At the end of the day, we obtain
\begin{equation}
\norm{L(\Asmooth_t) - L(P_t)} \lesssim \frac{\commratio}{\ntaumin \alpha_n} (\err_1(\betamax) + \err_2(\betamax))
\end{equation}
Which is minimized for the same choice of $\betamax \sim \rho_n$.

\section{Conclusion and outlooks}

In the DSBM, it should come as no surprise that a model that is very regular should not need to be as dense as when treading with a single snapshot. Our analysis is the first to show this, for classic SC, in a non-asymptotic manner. Under a slightly stronger condition on the regularity than that in \cite{Pensky2019}, we showed that strong consistency guarantees can be obtained even in the sparse case. We extended the results to the normalized Laplacian and, although we obtain the same final error rate as the adjacency matrix, our analysis also yields, to our knowledge, the best non-asymptotic spectral bound concentration of the normalized Laplacian for Bernoulli matrices with independent edges.

In this theoretical paper, we did not discuss how to select in practice the various parameters of the algorithms such as the number of communities $K$ or the forgetting factor $\lambda$. This is left for future investigations, as well as the analysis of varying $K$, $n$, or $B$. As we mentioned in Remark \ref{rem:constant_eps}, an outstanding conjecture about the sparse case and $\varepsilon_n \sim 1$ is formulated in \cite{Ghasemian2016}. Finally, our new spectral concentration of the normalized Laplacian, which shows that $\norm{L(A) - L(P)} \to 0$ in the relatively sparse case, may have consequences in other asymptotic analyses of the spectral convergence of the normalized Laplacian \cite{VonLuxburg2008, Tang2018a, Levie2019}.

\section*{Acknowledgements}
S. Vaiter was supported by ANR GraVa ANR-18-CE40-0005 and Projet ANER RAGA G048CVCRB-2018ZZ. We thank Nicolas Verzelen for useful discussion and pointing us to references.

\bibliographystyle{plain}
\bibliography{bibdsbm_cleaned,library}

\appendix
\section{Proofs}

\subsection{Proof of Lemma~\ref{lem:reco}}\label{app:proof_reco}

    From \cite[Section 5.4]{Lei2015}, for \emph{any} matrix $M \in \RR^{K \times K}$ and $Q = \Theta M \Theta^\top$, given an estimator $\hat Q$ that we feed to the SC algorithm it holds that
    \[
    E(\hat \Theta, \Theta) \lesssim (1+\delta) \frac{n'_{\max} K}{n n_{\min}^2 \gamma_M^2}\norm{\hat Q - Q}^2
    \]
    where $\gamma_M$ is the smallest eigenvalue of $M$.
    
    When using the adjacency matrix $\hat P = A$ to estimate the probability matrix $P = \Theta B \Theta^\top$, we have $B = \alpha_n B_0$, and $\gamma_M=\alpha_n \gamma$, which gives us \eqref{eq:guaranteeP}. When $B_0$ is defined as \eqref{eq:Bdef}, we have $\gamma = 1-\tau$.
    
    In the Laplacian case, for a node $i\leq n$ belonging to a community $k\leq K$, we have $d_i = \sum_j p_{ij} = d'_k \eqdef (n_k-1) B_{kk} + \sum_{\ell\neq k} n_\ell B_{k\ell} \leq \alpha_n \ntaumax$, hence the Laplacian of the probability matrix $L(P)$ can be written as:
    \begin{align*}
    L(P) = D(P)^{-\frac12} P D(P)^{-\frac12} =D(P)^{-\frac12}\Theta B \Theta^\top D(P)^{-\frac12} = \Theta \left(D_B^{-\frac12} B D_B^{-\frac12} \right) \Theta^\top
    \end{align*}
    where $D_B = \diag(d'_k) \in \RR^{K\times K}$. Hence we can apply the result above with $M = D_B^{-\frac12} B D_B^{-\frac12}$, and since
    \[
    \norm{D_B^{\frac12} B^{-1} D_B^{\frac12}} \leq \frac{\alpha_n \ntaumax}{\alpha_n \gamma} \leq \frac{\ntaumax}{\gamma}
    \]
    the smallest eigenvalue of $D_B^{-\frac12} B D_B^{-\frac12}$ satisfies $\gamma_M \geq \frac{\gamma}{\ntaumax}$, which leads to the result.

\subsection{Proof of Theorem \ref{thm:concentration}}\label{app:proof_thm_concentration}

%

The proof is heavily inspired by \cite{Lei2015}.
Define $P_k = \mathbb{E}(A_k)$, $W_k = A_k - P_k$ and $w_{ij}^{(k)}$ its elements, and their respective smoothed versions $A=\sum_{k=0}^t \beta_k A_{t-k}$, $P = \mathbb{E}(A)$, $W = A-P$, and $a_{ij}$, $p_{ij}$, $w_{ij}$ their elements.
Denote by $S$ the Euclidean ball in $\RR^n$ of radius $1$.
The proof strategy of \cite{Lei2015} is to define a grid
\[
T = \set{x \in S: 2\sqrt{n} x_i \in \mathbb{Z}}
\]
and simply note that (Lemma 2.1 in \cite{Lei2015} supplementary):
\[
\norm{W} = \sup_{u \in S} |u^\top W u| \leq 4\sup_{x,y \in T} \abs{x^\top W y}
\]
Hence we must bound this last quantity. To do this, for each given $(x,y)$ in $T$, we divide their indices into "light" pairs:
\[
\mathcal{L}(x,y) = \set{ (i,j): \abs{x_i y_j}\leq \sqrt{\frac{\alpha_n}{\betamax n}}}
\]
and "heavy" pairs $\mathcal{H}(x,y)$ are all the other indices. We naturally divide
\[
\sup_{x,y\in T} \abs{x^\top W y} \leq \sup_{x,y\in T} \abs{\sum_{(i,j) \in \mathcal{L}(x,y)} x_i y_j w_{ij}} + \sup_{x,y\in T} \abs{\sum_{(i,j) \in \mathcal{H}(x,y)} x_i y_j w_{ij}}
\]
and bound each of these two terms separately.

\subsubsection{Bounding the light pairs}

To bound the light pairs, Bernstein's concentration inequality is sufficient.

\begin{lemma}[Bounding the light pairs]\label{lem:light-pairs}
We have
\[
\mathbb{P}\left( \sup_{x,y\in T} \abs{\sum_{(i,j) \in \mathcal{L}(x,y)} x_i y_j w_{ij}} \geq c\sqrt{n \alpha_n \betamax}\right) \leq 2 e^{-\left(\frac{c^2/2}{2\Cbeta + \frac{2}{3} c}- \log(14)\right) n}
\]
for all constants $c>0$.
\end{lemma}
\begin{proof}
The proof is immediate by applying Bernstein's inequality. Take any $(x,y) \in T$, denote $C =\sqrt{\frac{\alpha_n}{\betamax n}}$. Define $u_{ij} = x_iy_j 1_{(i,j) \in \mathcal{L}(x,y)} + x_j y_i 1_{(j,i) \in \mathcal{L}(x,y)}$ (which is necessary because the edges $(i,j)$ and $(j,i)$ are not independent). We have
\[
\sum_{(i,j) \in \mathcal{L}(x,y)} x_iy_j w_{ij} = \sum_{1\leq i < j \leq n} u_{ij} w_{ij} = \sum_{1\leq i < j \leq n}  \sum_{k=0}^t u_{ij} \beta_k w^{(t-k)}_{ij}
\]
where $w^{(t-k)}_{ij}$ is a centered Bernoulli variable of parameter $p^{(t-k)}_{ij}$. Hence for all $i<j$, $0\leq k\leq t$, let $Y_{ijk} = \beta_k u_{ij} w^{(t-k)}_{ij}$ be independent random variables such that $\mathbb{E}Y_{ijk} = 0$, $\abs{Y_{ijk}} \leq 2 C \betamax$ since $u_{ij} \leq 2C$, $\beta_k \leq \betamax$ and $w_{ij}^{(t-k)} \leq 1$, and $\sigma_{ijk}^2 := \text{Var}(Y_{ijk}) \leq 2\beta_k^2 (x_i^2 y_j^2 + x_j^2y_i^2) \alpha_n$ since $\text{Var}(w^{(t-k)}_{ij}) = p^{(t-k)}_{ij}(1-p^{(t-k)}_{ij}) \leq \alpha_n$. Note that
\begin{align*}
\sum_{i<j} \sum_k \sigma_{ijk}^2 &\leq 2\alpha_n (\sum_{i<j} x_i^2y_j^2 + x_j^2y_i^2) \cdot (\sum_{k=0}^t \beta_k^2) \\
&\leq 2\alpha_n (\sum_{i=1}^n x_i^2)(\sum_{j=1}^n y_j^2) \Cbeta \betamax \\
&\leq 2\Cbeta \alpha_n \betamax
\end{align*}
Hence, applying Bernstein's inequality:
\begin{align*}
\mathbb{P}\left( \abs{\sum_{(i,j) \in \mathcal{L}(x,y)} x_iy_j w_{ij}} \geq t\right) &\leq 2 \exp\left(-\frac{t^2/2}{2\Cbeta \alpha_n\betamax + \frac{2}{3} C \betamax t}\right)\\
\mathbb{P}\left( \abs{\sum_{(i,j) \in \mathcal{L}(x,y)} x_iy_j w_{ij}} \geq c\sqrt{\alpha_n\betamax n}\right) &\leq 2 \exp\left(-\left(\frac{c^2/2}{2\Cbeta + \frac{2}{3} c}\right) n\right)
\end{align*}
Then, we use the fact that $\abs{T} \leq e^{n\log(14)}$ (see proof of Lemma 3.1 in \cite{Lei2015}) and a union bound to conclude.
\end{proof}

\subsubsection{Bounding the heavy pairs}

To bound the heavy pairs, two main Lemmas are required: the so-called \emph{bounded degree} (Lemma \ref{lem:bounded-degree}) and \emph{bounded discrepancy} lemma, presented below. As mentioned before, the bounded degree lemma is key in improving the sparsity hypothesis, despite the simplicity of its proof. The bounded discrepancy lemma is closer to its original proof \cite{Lei2015}, that we reproduce here for completeness.


\begin{lemma}[Bounded discrepancy]\label{lem:bounded-discrepancy}
For $I,J \subset \{1,\ldots,n\}$, we define
\[
\mu(I,J) = \alpha_n \abs{I}\abs{J},\quad e(I,J) = \sum_{k=0}^t \beta_k e_{t-k}(I,J)
\]
where $e_t(I,J)$ is the number of edges between $I$ and $J$ at time $t$. Then, for all $c, c'$, with probability $1-e^{- \frac{c^2/2}{\Cbeta +2c/3} \cdot \frac{n \alpha_n}{\betamax}+ \log n} - n^{-\frac{c'}{4}+6}$: for all $\abs{I}\leq \abs{J}$ at least one the following is true:
\begin{enumerate}
\item $e(I,J) \leq c'' \mu(I,J)$ with $c'' =\max(ec, 8)$
\item $e(I,J) \log\frac{e(I,J)}{\mu(I,J)} \leq c' \betamax \abs{J} \log \frac{n}{\abs{J}}$
\end{enumerate}
\end{lemma}
Of course by symmetry it is also valid for $\abs{J}\leq \abs{I}$ with the same probability (inverting the role of $I$ and $J$ in the bounds).

To prove it, we need the following Lemma.
\begin{lemma}[Adapted from Lemma 9 in \cite{Pensky2019}]\label{lem:alon}
Let $X_1,..., X_n$ be independent variables such that $X_i = Y_i - \mathbb{E}Y_i$, where $Y_i$ is a Bernoulli random variable with parameter $p_i$. Define $X = \sum_i w_i X_i$ where $0 \leq w_i \leq w_{\max}$. Let $\mu$ be such that $\sum_{i=1}w_i p_i \leq \mu$. Then, for all $t\geq 7$, we have
\[
\mathbb{P}\left(X \geq t\mu\right) \leq \exp\left(-\frac{t\log(1+t) \mu}{2 w_{\max}}  \right)
\]
\end{lemma}
\begin{proof}
For some $\lambda>0$ to be fixed later, have $\mathbb{E}(e^{\lambda w_i X_i}) = p_i e^{w_i (1-p_i)\lambda} + (1-p_i)e^{w_i p_i \lambda}$. Hence
\begin{align*}
\mathbb{E}(e^{\lambda X}) = \prod_i \mathbb{E}(e^{\lambda w_i X_i}) &= \prod_i \left(p_i e^{w_i (1-p_i)\lambda} + (1-p_i)e^{w_i p_i \lambda}\right) \\
&\leq e^{- \lambda \sum_i w_i p_i} \prod_i \left(1 + p_i(e^{w_i \lambda} - 1)\right)
\end{align*}
Using $1+a \leq e^a$ and $e^x - 1 \leq \frac{e^A - 1}{A}x$ for $0 \leq x \leq A$, we have
\[
\prod_i \left(1 + p_i(e^{w_i \lambda} - 1)\right) \leq \exp\left(\frac{e^{w_{\max}\lambda} - 1}{w_{\max}} \sum_i w_i p_i\right)
\]
Hence, for $t\geq 7$ and $\lambda = \frac{\log(1+t)}{w_{\max}}$,
\begin{align*}
\mathbb{P}(X \geq t\mu) &\leq e^{-t\mu\lambda} \mathbb{E}(e^{\lambda X}) \leq \exp\left(\left(\frac{e^{w_{\max}\lambda} - 1}{w_{\max}} - \lambda\right) \sum_i w_i p_i - \lambda t\mu\right) \\
&= \exp\left(\frac{1}{w_{\max}}\left((t - \log(1+t)) \sum_i w_i p_i - \log(1+t) t\mu\right)\right) \\
&\leq \exp\left(\frac{\mu}{w_{\max}}\left(t - \log(1+t) - \log(1+t) t\right)\right) \quad \text{since $\log(1+t)\leq t$ and $\sum_i w_i p_i \leq \mu$} \\
&\leq \exp\left(\frac{\mu}{2w_{\max}}t \log(1+t)\right)
\end{align*}
since $t - \log(1+t) \leq \frac{1}{2}t \log(1+t)$.
\end{proof}

The Lemma above is slightly stronger than Bernstein in this particular case: we would have obtained $O(t)$ instead of $O(t \log(1+t))$.
Now we can prove the bounded discrepancy lemma.

\begin{proof}[Proof of Lemma \ref{lem:bounded-discrepancy}]
We assume that the bounded degree property (Lemma \ref{lem:bounded-degree}) holds, which implies that for all $I,J$, it hold that:
\begin{align*}
e(I,J) &= \sum_{k=0}^t \beta_k e_{t-k}(I,J) \leq \sum_k  \beta_k \min\left(\sum_{i \in I} d_{i,t-k},~\sum_{j\in J}d_{j,t-k}\right) \\
&\leq \min\left(\sum_{i \in I} d_i,~\sum_{j\in J}d_j\right) \leq c n \alpha_n \min(\abs{I},\abs{J})
\end{align*}

Following this, for any pair $I,J$ such that $\abs{I}\geq n/e$ or $\abs{J}\geq n/e$ (where $e=\exp(1)$ is chosen for later conveniency), then $\frac{e(I,J)}{\mu(I,J)} \leq \frac{c n \alpha_n \min(\abs{I},\abs{J})}{\alpha_n |I| |J|} \leq c e$ and the result is proved.

Thus we now considers the pairs $I,J$ where both have size less than $n/e$, and such that $|I|\leq |J|$ without lost of generality. For such a given pair $I,J$, we decompose
\[
e(I,J) = \sum_{k=0}^t \sum_{(i,j)} \beta_k a_{ij}^{(t-k)} = \sum_{i,j,k} Y_{ijk}
\]
where the sum over $(i,j)$ counts only once each distinct edge between $I$ and $J$, and $Y_{ijk} = a_{ij}^{(t-k)}$ is a Bernoulli variable with parameter $p_{ij}^{(t-k)}$. Using $\sum_{i,j,k} \beta_k p_{ij}^{(t-k)} \leq |I||J| \alpha_n = \mu(I,J)$ and Lemma \ref{lem:alon}, we have, for any $t\geq 8$,
\begin{align*}
\mathbb{P}(e(I,J) \geq t \mu(I,J)) &\leq \mathbb{P}\left(e(I,J) - \mathbb{E}e(I,J) \geq (t-1)\mu(I,J)\right) \\
&\leq \exp\left(-\frac{ \mu(I,J) (t-1)\log k }{2\betamax}\right) \leq \exp\left(-\frac{ \mu(I,J) t\log t }{4\betamax}\right)
\end{align*}
Denoting $u=u(I,J)$ the unique value such that $u\log u = \frac{c' \betamax |J|}{\mu(I,J)}\log\frac{n}{|J|}$ and $t(I,J) = \max(8,u(I,J))$, we have (again for a fixed pair $I,J$ of size less that $n/e$):
\[
\mathbb{P}(e(I,J) \geq t(I,J) \mu(I,J)) \leq e^{-\frac{c'}{4} |J| \log\frac{n}{|J|}}
\]

Then, performing the same computations as in \cite{Lei2015} (reproduced here for completeness):
\begin{align*}
\mathbb{P}&\left(\exists I,J: |I|\leq |J| \leq \frac{n}{e},~ e(I,J) \geq t(I,J) \mu(I,J)\right) \\
&\leq \sum_{1\leq |I| \leq |J|\leq n/e} e^{-\frac{c'}{4} |J| \log\frac{n}{|J|}} = \sum_{1\leq h \leq g\leq n/e} \binom{n}{h}\binom{n}{g} e^{-\frac{c'}{4} g \log\frac{n}{g}} \\
&\leq \sum_{1\leq h \leq g\leq n/e} \left(\frac{ne}{h}\right)^h \left(\frac{ne}{g}\right)^g e^{-\frac{c'}{4} g \log\frac{n}{g}} \\
&\leq \sum_{1\leq h \leq g\leq n/e} \exp\left(h \log\frac{ne}{h} + g \log\frac{ne}{g} - \frac{c'}{4} g \log\frac{n}{g}\right) \\
&\leq \sum_{1\leq h \leq g\leq n/e} \exp\left(4g \log\frac{n}{g} - \frac{c'}{4} g \log\frac{n}{g}\right) \leq \sum_{1\leq h \leq g\leq n/e} n^{-\frac{c'}{4} + 4} \leq n^{-\frac{c'}{4} + 6}
\end{align*}
using the fact that $x\to x\log x$ is increasing on $[1, n/e]$. So, $e(I,J) \geq t(I,J) \mu(I,J)$ holds uniformly for all pairs $I,J$ with high probability.

Finally, we distinguish two cases depending on the value of $t(I,J)$. If $t(I,J)=8$ we get $e(I,J) \leq 8 \mu(I,J)$. If $t(I,J) = u(I,J) \geq 8$, we have $e(I,J) \leq u \mu(I,J)$, and
\[
\frac{e(I,J)}{\mu(I,J)}\log \frac{e(I,J)}{\mu(I,J)} \leq u \log u \leq \frac{c' \betamax}{\mu(I,J)}|J| \log\frac{n}{|J|}
\]
\end{proof}

We can now prove the bound on the heavy pairs, that is, we want to prove with high probability:
\[
\sup_{x,y\in T} \left| \sum_{(i,j) \in \mathcal{H}(x,y)} x_i y_j w_{ij}\right| \lesssim \sqrt{n \alpha_n \betamax}
\]
Since $p_{ij}\leq \alpha_n$ and by definition of the heavy pairs, for all $x,y \in T$:
\[
\left| \sum_{(i,j) \in \mathcal{H}(x,y)} x_i y_j p_{ij}\right| \leq \alpha_n \sum_{(i,j) \in \mathcal{H}(x,y)} \frac{x_i^2 y_j^2}{\abs{x_iy_j}} \leq \alpha_n \sqrt{\frac{n \betamax}{\alpha_n}} \norm{x}^2 \norm{y}^2 \leq \sqrt{n \alpha_n \betamax}
\]
Hence our goal is now to bound $\sup_{x,y\in T} \left| \sum_{(i,j) \in \mathcal{H}(x,y)} x_i y_j a_{ij}\right|$. We will show that, when the bounded degree and bounded discrepancy properties hold, this sum is bounded for all $x,y$. From now on, we assume that these results hold, and consider any $x,y \in T$.
Let us define sets of indices $I_s,J_t$ over which we bound uniformly $x_i$ and $y_j$, and replace the sum over $a_{ij}$ is these sets by $e(I_s,J_t)$. More specifically, we define
\begin{align*}
I_s &= \left\lbrace i: 2^{s-1}\frac{1/2}{\sqrt{n}} \leq \abs{x_i} < 2^s\frac{1/2}{\sqrt{n}}\right\rbrace \quad \text{for}\quad s=1,\ldots,\log_2(2\sqrt{n})+1 \\
J_t &= \left\lbrace j: 2^{t-1}\frac{1/2}{\sqrt{n}} \leq \abs{y_j} < 2^t\frac{1/2}{\sqrt{n}}\right\rbrace \quad \text{for}\quad t=1,\ldots,\log_2(2\sqrt{n})+1
\end{align*}
Since we consider heavy pairs, we need only consider indices $(s,t)$ such that $2^{s+t} \geq 16 \sqrt{\frac{\alpha_n n}{\betamax}}$, and we define $C_n = \sqrt{\frac{\alpha_n n}{\betamax}}$ for convenience. Moreover, we have $\sum_{i\in I_s, j\in J_t} a_{ij} \leq 2 e(I_s, J_t)$, since each edge indices appears at most twice. Hence, we have:
\begin{equation}\label{eq:proof_main_thm_maineq}
\left| \sum_{(i,j) \in \mathcal{H}(x,y)} x_i y_j a_{ij}\right| \leq \sum_{(s,t): 2^{s+t} \geq 16 C_n} 2^{s+t}\frac{1/4}{n} \cdot 2e(I_s, J_t)
\end{equation}

We now introduce more notations. We denote $\mu_s = \frac{4^s\abs{I_s}}{n}, \nu_t = \frac{4^t\abs{J_t}}{n}$, $\gamma_{st} = \frac{e(I_s, J_t)}{\alpha_n \abs{I_s}\abs{J_t}}$, $\sigma_{st} = \gamma_{st} C_n 2^{-(s+t)}$. We reformulate \eqref{eq:proof_main_thm_maineq} as:
\begin{align}
\left| \sum_{(i,j) \in \mathcal{H}(x,y)} x_i y_j a_{ij}\right| &\leq \sum_{(s,t): 2^{s+t} \geq 16 C_n} 2^{s+t}\frac{1/2}{n} \cdot e(I_s, J_t) \notag\\
&= \tfrac{1}{2}\sqrt{n \alpha_n \betamax} \sum_{s,t} \tfrac{1}{\sqrt{n \alpha_n \betamax}}2^{-(s+t)} 4^s 4^t \tfrac{1}{n} \cdot \tfrac{e(I_s,J_t)\alpha_n \abs{I_s}\abs{J_t}}{\alpha_n \abs{I_s}\abs{J_t}}\cdot \tfrac{\sqrt{n}}{\sqrt{n}} \notag\\
&= \frac12 \sqrt{n \alpha_n \betamax} \sum_{s,t} \mu_s \nu_t \sigma_{st}
\end{align}
Our goal is therefore to show that $\sum_{s,t} \mu_s \nu_t \sigma_{st} \lesssim 1$. For this, we will make extensive use of the fact that $\mu_s \leq 16 \sum_{i \in I_s} x_i^2$, and therefore $\sum_s \mu_s \leq 16$, and similarly $\sum_t \nu_t \leq 16$.

Following the original proof of \cite{Lei2015}, let $\mathcal{C} = \{(s,t): 2^{s+t} \geq 16C_n, \abs{I_s}\leq \abs{J_t}\}$, divided in six:
\begin{align*}
\mathcal{C}_1 &= \{(s,t) \in \mathcal{C}: \sigma_{st} \leq 1\} \\
\mathcal{C}_2 &= \{(s,t) \in \mathcal{C} \setminus \mathcal{C}_1 : \gamma_{st} \leq c_2\} \\
\mathcal{C}_3 &= \{(s,t) \in \mathcal{C} \setminus (\cup_{i=1}^2 \mathcal{C}_i) : 2^{s-t} \geq C_n\} \\
\mathcal{C}_4 &= \{(s,t) \in \mathcal{C} \setminus (\cup_{i=1}^3 \mathcal{C}_i) : \log \gamma_{st} > \tfrac{1}{4} \log\tfrac{4^t}{\nu_t}\} \\
\mathcal{C}_5 &= \{(s,t) \in \mathcal{C} \setminus (\cup_{i=1}^4 \mathcal{C}_i) : 2t\log 2 \geq \log\tfrac{1}{\nu_t}\} \\
\mathcal{C}_6 &= \{(s,t) \in \mathcal{C} \setminus (\cup_{i=1}^5 \mathcal{C}_i)\}
\end{align*}
Similarly, we define $\mathcal{C}' = \{(s,t): 2^{s+t} \geq 16C_n, \abs{I_s}\geq \abs{J_t}\}$ and $\mathcal{C}_i'$ the same way by inverting the roles of $\mu_s$ and $\nu_t$. We write the proof for $\mathcal{C}$, the other case is strictly symmetric.
Our goal is to prove that each of the $\sum_{(s,t) \in \mathcal{C}_i} \mu_s \nu_t \sigma_{st}$ is bounded by a constant.

\paragraph*{Pairs in $\mathcal{C}_1$} In this case we get
\[
\sum_{(s,t) \in \mathcal{C}_1} \mu_s \nu_t \sigma_{st} \leq \sum_{s,t} \mu_s \nu_t \leq 16^2
\]

\paragraph*{Pairs in $\mathcal{C}_2$} This includes the indices for which the first case in the bounded discrepancy lemma (Lemma \ref{lem:bounded-discrepancy}) is satisfied. Since for $s,t \in \mathcal{C}$ we have $2^{s+t} \geq 16 C_n$, we have $\sigma_{st} \leq \gamma_{st}/16$, and
\[
\sum_{(s,t) \in \mathcal{C}_2} \mu_s \nu_t \sigma_{st} \leq c''\sum_{s,t} \mu_s \nu_t/16 \leq c'' 16
\]

\paragraph*{Pairs in $\mathcal{C}_3$.} Since $2^{s-t} \geq C_n$, we have necessarily $t \leq s - \log_2 C_n$. Furthermore, since we assumed the bounded degree property (Lemma \ref{lem:bounded-degree}), we have $e(I_s,J_t) \leq c \abs{I_s} \alpha_n n$, and therefore $\gamma_{st} \leq cn/\abs{J_t}$. Then,
\begin{align*}
\sum_{(s,t) \in \mathcal{C}_3} \mu_s \nu_t \sigma_{st} &\leq \sum_{s} \mu_s \sum_{t=1}^{s-\log_2 C_n} \frac{4^t \abs{J_t}}{n} \frac{cn}{\abs{J_t}} C_n 2^{-(s+t)} \\
&= c\sum_{s} \mu_s C_n 2^{-s} \sum_{t=1}^{s-\log_2 C_n} 2^t \\
&= c\sum_{s} \mu_s C_n 2^{-s} (2^{s-\log_2 C_n + 1} -1) \\
&\leq 2c \sum_s \mu_s \leq 32c
\end{align*}

\paragraph*{Pairs in $\mathcal{C}_4$} All $I_s,J_t$ for which the first case of the bounded discrepancy lemma is satisfied are included in $\mathcal{C}_2$, hence the remaining sets satisfy the second case of the lemma, which reads $e(I_s,J_t) \log \gamma_{st} \leq c' \betamax \abs{J_t} \log \frac{n}{\abs{J_t}}$. It can be reformulated as
\begin{align}
\gamma_{st} \alpha_n \abs{I_s}\abs{J_t} \log \gamma_{st} &\leq c' \betamax \abs{J_t} \log\frac{4^t}{\nu_t} \notag \\
\sigma_{st} 2^{s+t} \sqrt{\frac{\betamax}{\alpha_n n}} \alpha_n \mu_s 4^{-s} n  \log \gamma_{st} &\leq c' \betamax \log\frac{4^t}{\nu_t} \notag \\
\sigma_{st} \mu_s \log\gamma_{st} &\leq c' \frac{2^{s-t}}{C_n} \log\frac{4^t}{\nu_t} \label{eq:proof_main_thm_bounded_discrepancy_2}
\end{align}
Since $(s,t) \notin \mathcal{C}_3$, we have $2^{s-t}\leq C_n$, and therefore $s \leq t+ \log_2 C_n$. Since $(s,t) \in \mathcal{C}_4$ we have $\log \gamma_{st} \geq \frac{1}{4} \log \frac{4^t}{\nu_t}$, and \eqref{eq:proof_main_thm_bounded_discrepancy_2} implies $\sigma_{st} \mu_s \leq 4c' \frac{2^{s-t}}{C_n}$. Then,
\begin{align*}
\sum_{(s,t) \in \mathcal{C}_4} \mu_s \nu_t \sigma_{st} \leq \sum_t \nu_t \sum_{s=1}^{t+\log_2 C_n} 4 c' \frac{2^{s-t}}{C_n} \leq 4c' \sum_t\nu_s \frac{2^{-t}}{C_n} (2^{t+\log_2 C_n + 1} -1) \leq 128 c'
\end{align*}

\paragraph*{Pairs in $\mathcal{C}_5$} We have $\frac{1}{\nu_t} \leq 4^t$, and since $(s,t) \notin \mathcal{C}_4$, we have $\log \gamma_{st} \leq \frac{1}{4} \log\frac{4^t}{\nu_t} \leq t \log 2$ and $\gamma_{st} \leq 2^t$. On the other hand, since $(s,t) \notin \mathcal{C}_1$, $1\leq \sigma_{st} = \gamma_{st} C_n 2^{-(s+t)} \leq C_n 2^{-s}$, and $s \leq \log_2 C_n$.

Because $(s,t) \notin \mathcal{C}_2$ and $c'' \geq 8$, $\log\gamma_{st} \geq 3\log 2$, combining with $\frac{1}{\nu_t} \leq 4^t$, equation \eqref{eq:proof_main_thm_bounded_discrepancy_2} becomes:
\[
\sigma_{st} \mu_s \leq c' \frac{2^{s-t}}{C_n} \frac{4t}{3} \leq c' \frac{2^{s+1}}{3C_n}
\]
since $t2^{-t} \leq 1/2$.

Combining these two facts,
\begin{align*}
\sum_{(s,t) \in \mathcal{C}_5} \mu_s \nu_t \sigma_{st} \leq \sum_t \nu_t \sum_{s=1}^{\log_2 C_n} \frac{2}{3C_n} c' 2^s \leq c' \sum_t\nu_s \frac{2}{3C_n} (2^{\log_2 C_n + 1} -1) \leq \frac{64}{3} c'
\end{align*}

\paragraph*{Pairs in $\mathcal{C}_6$} Finally, we have $0 \leq \log \gamma_{st} \leq \frac12 \log \frac{1}{\nu_t} \leq \log \frac{1}{\nu_t}$ because, respectively, $(s,t) \notin \mathcal{C}_2$, $(s,t) \notin \mathcal{C}_4$ and $(s,t) \notin \mathcal{C}_5$, so $ \gamma_{st}\leq \frac{1}{\nu_t}$. Since by definition $2^{s+t} \geq C_n$, we have $t \geq \log_2 C_n -s$, and
\begin{align*}
\sum_{(s,t) \in \mathcal{C}_6} \mu_s \nu_t \sigma_{st} &= \sum_{(s,t) \in \mathcal{C}_6} \mu_s \nu_t \gamma_{st} 2^{-(s+t)} C_n \\
&\leq \sum_s \mu_s C_n 2^{-s}\sum_{t\geq \log_2 C_n -s} (1/2)^t \\
&= \sum_s \mu_s C_n 2^{-s}\left(2 - \frac{1-(1/2)^{\log_2 C_n - s}}{1-(1/2)}\right) = 2\sum_s \mu_s \leq 32
\end{align*}

We conclude the proof of Theorem \ref{thm:concentration} by gathering the bounds on the light and heavy pairs, with the corresponding probabilities of failure. We consider the same constant $c>0$ in each lemma for simplicity.

\subsection{Concentration of Laplacian: proof of Theorem \ref{thm:concentration_Lapl}}\label{app:Lapl}

We note the degree matrices of $A$ and $P$ respectively $D$ and $D_{P}$, containing the degrees $d_i = \sum_{kj}\beta_k a_{ij}^{(t-k)}$ and $\bar d_i = \mathbb{E}d_i$. Note that under our assumptions $d_{\min} \eqdef \alpha_n \ntaumin \leq \bar d_i \leq d_{\max} \eqdef \alpha_n \ntaumax$. Applying Lemma \ref{lem:bounded-degree-Lapl} with $c=\frac{1}{2}$, we obtain: for all $\nu>0$, there is a constant $C'_\nu$ such that, if $\frac{\alpha_n}{\betamax} \geq C'_\nu \commratio \frac{\log n}{\ntaumin}$, with probability at least $1-n^{-\nu}$ we have $\frac{1}{2} d_{\min} \leq d_i \leq \frac{3}{2} d_{\max}$ for all $i$. We assume that it is satisfied for the rest of the proof.

%
We apply Lemma \ref{lem:technicalLaplacianBound}, from which
\begin{equation}\label{eq:proofdecomposeLaplacian}
\norm{L(A)-L(P)} \leq \frac{2\norm{A-P}}{d_{\min}}
+\frac{4\norm{(D-D_P)P}}{d_{\min}^2}
\end{equation}
We will now bound $\norm{A-P}$ and $\norm{(D-D_P)P}$ with high probability, and use a union bound to conclude.

By Theorem \ref{thm:concentration}, with probability $1-n^{-\nu}$ we have $\norm{A-P}\lesssim \sqrt{n \alpha_n \betamax}$ and the first term has the desired rate.

To bound the spectral norm of $(D-D_P)P$ with high probability, we re-use the ``light and heavy pairs'' strategy of the previous proof. Define $\delta_i = d_i - \bar d_i$. We adopt the definitions of the previous section. We use again the fact that
\[
\norm{(D-D_P)P} \leq 4 \sup_{x,y \in T} x^\top Q y
\]
where $T$ is the same grid.
We decompose
\[
x^\top Q y = \sum_{i,j \in \mathcal{L}(x,y)} x_i y_j p_{ij} \delta_i + \sum_{i,j \in \mathcal{H}(x,y)} x_i y_j p_{ij} \delta_i
\]
Recall that we have $\abs{\delta_i} \leq d_{\max}$. In the proof of Theorem \ref{thm:concentration} we proved that $\sum_{i,j \in \mathcal{H}(x,y)} x_i y_j p_{ij} \leq \sqrt{n \alpha_n \betamax}$, and therefore with the same probability
\[
\sum_{i,j \in \mathcal{H}(x,y)} x_i y_j p_{ij} \delta_i \lesssim d_{\max} \sqrt{n \alpha_n \betamax}
\]
which is the desired complexity.

We must now handle the light pairs. We write $\delta_i = \sum_{\ell=1}^{n} \sum_{k=0}^t \beta_k w_{i\ell}^{(t-k)}$, and therefore
\[
\sum_{i,j \in \mathcal{L}(x,y)} x_i y_j p_{ij} \delta_i = \sum_{i< \ell} \sum_k u_{i \ell k} w_{i\ell}^{(t-k)}
\]
where
\[
u_{i \ell k} =\beta_k \sum_j \left(x_i y_j p_{ij} 1_{(i,j)\in \mathcal{L}(x,y)} + x_\ell y_j p_{\ell j} 1_{(\ell,j)\in \mathcal{L}(x,y)}\right)
\]
We want to apply Bernstein inequality. The random variables $w_{i\ell}^{(t-k)}$ are independent centered Bernoulli variables of parameters $p_{i\ell}^{(t-k)}$. By definition of light pairs we have
\[
\abs{u_{ijk}} \leq 2 \betamax \alpha_n \ntaumax  \sqrt{\frac{\alpha_n}{n \betamax}} = 2\sqrt{\betamax} \alpha_n^{\frac{3}{2}} \frac{\ntaumax}{\sqrt{n}}
\]
since $\sum_j p_ij^{(t-k)} \leq \alpha_n \ntaumax$.
Then, using $(a+b)^2 \leq 2(a^2+b^2)$,
\begin{align*}
\sum_{i \ell k} Var(u_{i\ell k} w_{i\ell}^{(t-k)}) &\leq \sum_{k} \beta_k^2 \left(\sum_{i\ell} p_{i\ell}^{(t-k)} \left( x_i \sum_j y_j p_{ij}+ x_\ell \sum_j y_j p_{\ell j}\right)^2\right) \\
&\leq 2\sum_k \betamax^2 \left(\sum_{i\ell} p_{i\ell}^{(t-k)} x_i^2 \left(\sum_j y_j p_{ij}\right)^2+ p_{i\ell}^{(t-k)} x_\ell^2\left( \sum_j y_j p_{\ell j}\right)^2\right) \\
&\leq 2\Cbeta \betamax \ntaumax^2 \alpha_n^3
\end{align*}
Where we have used
\[
\sum_j y_j p_{\ell j} = \sum_k\beta_k \sum_j y_j p_{\ell j}^{(t-k)} \leq \sum_k \beta_k \norm{y} \alpha_n\sqrt{n_{\max} + \tau^2 n} \leq \alpha_n \sqrt{\ntaumax}
\]
and $\sum_{i\ell} p_{i\ell}^{(t-k)} x_i^2 \leq \alpha_n \ntaumax \norm{x}$.
Hence, using Bernstein's inequality,
\begin{align*}
\mathbb{P}\left(\abs{\sum_{i,j \in \mathcal{L}(x,y)} x_i y_j p_{ij} \delta_i} \geq t\right) &\leq 2\exp\left(-\frac{t^2/2}{2\Cbeta \betamax \ntaumax^2 \alpha_n^3 + \frac{2}{3}\sqrt{\betamax} \alpha_n^{\frac{3}{2}} \frac{\ntaumax}{\sqrt{n}} t}\right) \\
\mathbb{P}\left(\abs{\sum_{i,j \in \mathcal{L}(x,y)} x_i y_j p_{ij} \delta_i} \geq c \ntaumax \sqrt{n\betamax}\alpha_n^{\frac{3}{2}}\right) &\leq 2\exp\left(-\frac{c^2/2}{2\Cbeta + \frac{2c}{3}} \cdot n\right)
\end{align*}
Using a union bound over $T$ we can conclude.

\subsection{Additional proofs}\label{app:additional_proofs}

%
\begin{proof}[Proof of Lemma \ref{lem:error_second_markov}]
For any $k$, we have $p^{(t-k)}_{ij} = p^{(t)}_{ij}$ if $z_i^{t-k} = z_i^t$ and $z_j^{t-k} = z_j^t$, that is, if the nodes have not changed communities. Thus we write
\begin{align*}
\norm{P_{t-k} - P_t}_F^2 &= \sum_{i,j} (p^{(t-k)}_{ij} - p^{(t)}_{ij})^2\pa{1-\mathbf{1}_{\{z_i^{t-k} = z_i^t\}}\mathbf{1}_{\{z_j^{t-k} = z_j^t\}}} \\
&\leq 4\alpha_n^2 n^2 \pa{1 - \pa{\frac{1}{n}\sum_i \mathbf{1}_{\{z_i^{t-k} = z_i^t\}}}^2}
\end{align*}
Considering that $z_i^{t-k} = z_i^t$ at least when $z_i^{t-k} =z_i^{t-k+1} = \ldots = z_i^t$ and that this event happens with probability $(1-\varepsilon_n)^k$, we have that
\[
\mathbf{1}_{\{z_i^{t-k} = z_i^t\}} \geq a_i \sim \text{Ber}((1-\varepsilon_n)^k)
\]
where the $a_i$ are independent Bernoulli variables. By Hoeffding inequality, for some $\delta>0$ that we will fix later, we have
\[
\mathbb{P}\pa{\frac{1}{n}\sum_i a_i \leq (1-\varepsilon_n)^k -\delta} \leq 2 e^{-2 \delta^2 n} 
\]

and therefore with probability at least $1-\rho$
\begin{align*}
\norm{P_{t-k} - P_t}_F^2 \leq 4\alpha_n^2 n^2 \pa{1-\pa{(1-\varepsilon_n)^{k} -\delta}^2} \leq 8\alpha_n^2 n^2 \pa{1-(1-\varepsilon_n)^{k} +\delta}
\end{align*}
then using $1-x^k = (1-x)(1+x + \ldots + x^{k-1}) \leq (1-x)k$ for $\abs{x}\leq 1$ we get
\[
\norm{P_{t-k} - P_t}_F^2 \leq 4\alpha_n^2 n^2 \pa{1-\pa{(1-\varepsilon_n)^{k} -\delta}^2} \leq 8\alpha_n^2 n^2 \pa{\min(1, k\epsilon_n) +\delta}
\]
Then we choose $\delta \sim \epsilon_n \leq \min(1, k\epsilon_n)$ to conclude.
%
%
\end{proof}

\begin{proof}[Proof of Lemma \ref{lem:error_second_Lapl}]
Denote by $D = D(P_t)$, $\bar D = D(\Psmooth_t)$ the degree matrices of $P_t$ and $\Psmooth_t$, with $d_i$ and $\bar d_i$ their elements. By assumption, we have $d_i, \bar d_i \geq d_{\min} \eqdef \ntaumin \alpha_n$ for all $i$. Therefore, by applying Lemma \ref{lem:technicalLaplacianBound} we have 
\[
\norm{L(P_t) - L(\Psmooth_t)} \leq \frac{\norm{P_t - \Psmooth_t}}{d_{\min}} + \frac{\norm{(D-\bar D)\Psmooth_t}}{d_{\min}^2}
\]
From Lemma \ref{lem:error_second_deter}, we have $\norm{P_t -\Psmooth_t} \lesssim \Cbeta' \alpha_n \sqrt{\frac{n \ntaumax \varepsilon_n}{\betamax}}$.

For the second term, we define $D_{t-k}$ the diagonal degree matrix associated to $P_{t-k}$, such that
\begin{align*}
\norm{(D-\bar D)\Psmooth_t} &\leq \sum_{k=0}^t\beta_k \norm{(D-D_{t-k})\Psmooth_t}_F
\end{align*}

Denoting by $\bar p_{ij}$ the elements of $\Psmooth_t$, recall that $\sum_j \bar p^2_{ij} \leq \alpha_n^2 \ntaumax$, and using $(a+b)^2 \leq 2(a^2 + b^2)$ we have
\begin{align*}
\norm{(D-D_{t-k})\Psmooth_t}_F^2 &= \sum_{i} \left(\sum_\ell p^{(t)}_{i\ell} - p^{(t-k)}_{i\ell}\right)^2 \pa{\sum_j \bar p^2_{ij}} \\
&\leq \alpha_n^2 \ntaumax \sum_i\left(\sum_\ell p^{(t)}_{i\ell} - p^{(t-k)}_{i\ell}\right)^2 \\
&\leq \alpha_n^2 \ntaumax \sum_i \pa{ \sum_\ell \pa{\sqrt{p^{(t)}_{i\ell}} + \sqrt{p^{(t-k)}_{i\ell}}}^2} \pa{\sum_\ell \pa{\sqrt{p^{(t)}_{i\ell}} - \sqrt{p^{(t-k)}_{i\ell}}}^2} \\
&\leq 2\alpha_n^2 \ntaumax \sum_i \pa{ \sum_\ell p^{(t)}_{i\ell} + p^{(t-k)}_{i\ell}}  \pa{\sum_\ell \pa{\sqrt{p^{(t)}_{i\ell}} - \sqrt{p^{(t-k)}_{i\ell}}}^2} \\
&\leq 4 \alpha_n^3 \ntaumax^2 \norm{P_t^{\odot \frac12} - P_{t-k}^{\odot \frac12}}_F^2
\end{align*}
where $A^{\odot \frac12}$ indicates element-wise square root. Repeating the proof of Lemma \ref{lem:error_second_deter}, for two SBM connection probability matrices $P$ and $P'$ between which only the nodes belonging to a set $\mathcal{S}$ have change community, we have
\begin{align*}
\norm{P^{\odot \frac12} - (P')^{\odot \frac12}}_F^2 &= \sum_{i \in \mathcal{S}} \sum_j \pa{\sqrt{p_{ij}} - \sqrt{p'_{ij}}}^2 + \pa{\sqrt{p_{ji}} - \sqrt{p'_{ji}}}^2 \\
&\leq 4\sum_{i \in \mathcal{S}} \sum_j p_{ij}+p'_{ij} \leq 8 \abs{\mathcal{S}} \alpha_n \ntaumax
\end{align*}
Therefore, $\norm{P_t^{\odot \frac12} - P_{t-k}^{\odot \frac12}}_F^2 \lesssim \alpha_n \ntaumax n \min(1, k\varepsilon_n)$, and 
$
\norm{(D-D_{t-k})\Psmooth_t}_F \leq \alpha_n^2 \ntaumax^{\frac{3}{2}} \sqrt{n} \min(1, \sqrt{k \varepsilon_n})
$. We conclude using the hypothesis on $\sum_k \beta_k \min(1, \sqrt{k \varepsilon_n})$.
\end{proof}

\subsection{Technical Lemma}\label{app:technical_proof}

\begin{lemma}\label{lem:technicalLaplacianBound}
Let $A, P \in \mathbb{R}^{n\times n}$ be symmetric matrices containing non-negative elements, assume that $d_i=\sum_j a_{ij}$ and $d^P_i = \sum_j p_{ij}$ are strictly positive, define $D=\diag(d_i)$, $D_P=\diag(d^P_i)$, $d_{\min} = \min_i(d_i,d^P_i)$. Then,
\[
\norm{L(A) - L(P)} \leq \frac{\norm{A-P}}{d_{\min}} + \frac{\norm{(D-D_P)P}}{d_{\min}^2}
\]
\end{lemma}
\begin{proof}
Recall that $\norm{M} = \sup_{x,y \in \mathcal{S}} x^\top M y$, where $\mathcal{S}$ is the Euclidean unit ball. Denote $Q = (D^{-\frac12}-D_P^{-\frac12})P$.

We write
\begin{align*}
L(A) - L(P) &= D^{-\frac12} A D^{-\frac12} - D_{P}^{-\frac12} P D_{P}^{-\frac12}  \\
&= D^{-\frac12} (A-P) D^{-\frac12} + D^{-\frac12} P D^{-\frac12} - D_P^{-\frac12} P D_P^{-\frac12} \\
&= D^{-\frac12} (A-P) D^{-\frac12} + Q D^{-\frac12} - D_P^{-\frac12} Q^\top
\end{align*}
and thus
\[
\norm{L(A) - L(P)} \leq \frac{\norm{A-P}}{d_{\min}} + \frac{2\norm{Q}}{\sqrt{d_{\min}}}
\]
Then, for $x,y \in \mathcal{S}$, we have
\begin{align*}
x^\top Q y &= \sum_{ij} x_i y_j p_{ij} \left( \frac{1}{\sqrt{d_i}} - \frac{1}{\sqrt{d^P_i}}\right) = \sum_{ij} x_i y_j p_{ij} \left(\frac{d_i^P - d_i}{\sqrt{d_i d^P_i}(\sqrt{d_i} + \sqrt{d^P_i})}\right) \\
=& \sum_{ij} x_i y_j p_{ij} \delta_i(d_i^P - d_i)
\end{align*}
where $\delta_i \eqdef \frac{1}{\sqrt{d_i d^P_i}(\sqrt{d_i} + \sqrt{d^P_i})} \leq \frac12 d_{\min}^{-\frac{3}{2}}$.
%
Since the $p_{ij}$ are non-negative, the maximum over $x,y \in \mathcal{S}$ is necessarily reached when every term in the sum is non-negative, by choosing $y_j \geq 0$ and $x_i$ with the same sign as $d^P_i - d_i$. Hence, $\sup_{x,y \in \mathcal{S}}\sum_{ij} x_i y_j p_{ij} (d^P_i - d_i) = \sup_{x,y \in \mathcal{S}}\sum_{ij} \abs{x_i y_j p_{ij} (d^P_i - d_i)}$. Using this property,
\begin{align*}
\sup_{x,y \in \mathcal{S}}\sum_{ij} x_i y_j p_{ij} \delta_i (d^P_i - d_i) &\leq \sup_{x,y \in \mathcal{S}}\sum_{ij} \abs{x_i y_j p_{ij} \delta_i(d^P_i - d_i)} \\
&\leq \frac12 d_{\min}^{-\frac{3}{2}} \sup_{x,y \in \mathcal{S}}\sum_{ij} \abs{x_i y_j p_{ij} (d^P_i - d_i)} \\
&= \frac12 d_{\min}^{-\frac{3}{2}} \sup_{x,y \in \mathcal{S}}\sum_{ij} x_i y_j p_{ij} (d^P_i - d_i) \\
&= \frac12 d_{\min}^{-\frac{3}{2}} \norm{(D-D_P)P}
\end{align*}
which concludes the proof.
\end{proof}

\end{document}